\documentclass[sigconf]{aamas} 

\usepackage{balance} 
\usepackage{algorithm}
\usepackage[noend]{algorithmic}
\usepackage[utf8]{inputenc}
\usepackage[T1]{fontenc}

\usepackage{tikz}
\usepackage{graphicx}
\usetikzlibrary{positioning}
\usetikzlibrary{matrix}
\usepackage{caption}
\usepackage{subcaption}
\usepackage{parskip}
\usepackage{tabularx,ragged2e}
\usepackage{multicol}
\usepackage{multirow}
\usepackage{titlesec}
\usepackage{tabularray}
\usepackage{diagbox}
\usepackage{makecell} 
\definecolor{darkgreen}{RGB}{0,100,0}
\usepackage{setspace}
\usepackage{scalerel}
\usepackage{amsfonts}
\usepackage{enumitem}
\usepackage{amsmath}
\DeclareMathOperator*{\argmax}{argmax} 
\usepackage{amsthm}
\usepackage{mathrsfs}
\theoremstyle{plain}
\newtheorem{theorem}{Theorem}[section]

\newtheorem{lemma}[theorem]{Lemma}

\theoremstyle{definition}
\newtheorem{definition}[theorem]{Definition}

\newtheorem{assumption}[theorem]{Assumption}
\theoremstyle{remark}

\usepackage{commath}   
\usepackage{xcolor}
\usepackage{bbm}

\usepackage{bm}
\usepackage{float}
\usepackage{forest}
\usepackage{mathtools}
\usepackage[capitalize,noabbrev]{cleveref}
\usepackage{titlesec}

\makeatletter
\renewcommand\subsubsection{\@startsection{subsubsection}{3}{\z@}%
                                     {-3.25ex\@plus -1ex \@minus -.2ex}%
                                     {1.5ex \@plus .2ex}%
                                     {\normalfont\normalsize\bfseries}}
\graphicspath{ {./Image/} }
\setcopyright{ifaamas}
\acmConference[AAMAS '24]{Proc.\@ of the 23rd International Conference
on Autonomous Agents and Multiagent Systems (AAMAS 2024)}{May 6 -- 10, 2024}
{Auckland, New Zealand}{N.~Alechina, V.~Dignum, M.~Dastani, J.S.~Sichman (eds.)}
\copyrightyear{2024}
\acmYear{2024}
\acmDOI{}
\acmPrice{}
\acmISBN{}
\acmSubmissionID{415}

\title[AAMAS-2024 Formatting Instructions]{Analysing the Sample Complexity of Opponent Shaping}

\newcommand{\eq}{\footnotemark[1]}

\author{Kitty Fung\eq{}}
\affiliation{
\institution{University of Oxford}
\city{Oxford}
\country{United Kingdom}}
\email{kittyfung01@gmail.com}

\author{Qizhen Zhang$\eq{}$}
\affiliation{
  \institution{University of Oxford}
  \city{Oxford}
  \country{United Kingdom}}
\email{qizhen.zhang@eng.ox.ac.uk}

\author{Chris Lu}
\affiliation{
  \institution{University of Oxford}
  \city{Oxford}
  \country{United Kingdom}}
  \email{christopher.lu@exeter.ox.ac.uk}

\author{Jia Wan}
\affiliation{
  \institution{Massachusetts Institute of Technology}
  \city{Cambridge}
  \country{United States}}
    \email{jiawan@mit.edu}
    
\author{Timon Willi}
\affiliation{
  \institution{University of Oxford}
  \city{Oxford}
  \country{United Kingdom}}
  \email{timon.willi@eng.ox.ac.uk}
  
\author{Jakob Foerster}
\affiliation{
  \institution{University of Oxford}
  \city{Oxford}
  \country{United Kingdom}}
  \email{jakob.foerster@eng.ox.ac.uk}

\begin{abstract}

Learning in general-sum games often yields collectively sub-optimal results. Addressing this, \textit{opponent shaping} (OS) methods actively guide the learning processes of other agents, empirically leading to improved individual and group performances in many settings. 
Early OS methods use higher-order derivatives to shape the learning of co-players, making them unsuitable to shape multiple learning steps. Follow-up work, Model-free Opponent Shaping (M-FOS), addresses these by reframing the OS problem as a \textit{meta-game}. 
In contrast to early OS methods, there is little theoretical understanding of the M-FOS framework. 
Providing theoretical guarantees for M-FOS is hard because A) there is little literature on theoretical sample complexity bounds for meta-reinforcement learning B) M-FOS operates in continuous state and action spaces, so theoretical analysis is challenging. 
In this work, we present R-FOS, a tabular version of M-FOS that is more suitable for theoretical analysis. 
R-FOS discretises the continuous meta-game MDP into a tabular MDP. Within this discretised MDP, we adapt the $R_{max}$ algorithm, most prominently used to derive PAC-bounds for MDPs, as the meta-learner in the R-FOS algorithm. We derive a sample complexity bound that is exponential in the cardinality of the inner state and action space and the number of agents. Our bound guarantees that, with high probability, the final policy learned by an R-FOS agent is close to the optimal policy, apart from a constant factor. Finally, we investigate how R-FOS's sample complexity scales in the size of state-action space. Our theoretical results on scaling are supported empirically in the Matching Pennies environment.
\footnotetext[1]{Equal Contribution \\ Work done at the Foerster Lab For AI Research (FLAIR), University of Oxford \\ Corresponding Authors: kittyfung01@gmail.com, qizhen.zhang@eng.ox.ac.uk \\ }
\end{abstract}


\newcommand{\BibTeX}{\rm B\kern-.05em{\sc i\kern-.025em b}\kern-.08em\TeX}

\makeatletter
\gdef\@copyrightpermission{
	\begin{minipage}{0.3\columnwidth}
		\href{https://creativecommons.org/licenses/by/4.0/}{\includegraphics[width=0.90\textwidth]{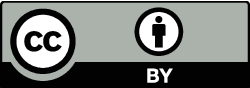}}
	\end{minipage}\hfill
	\begin{minipage}{0.7\columnwidth}
		\href{https://creativecommons.org/licenses/by/4.0/}{This work is licensed under a Creative Commons Attribution International 4.0 License.}
	\end{minipage}
	\vspace{5pt}
}
\makeatother

\begin{document}

\pagestyle{fancy}
\fancyhead{}

\maketitle 




\section{Introduction}


Learning in general-sum games commonly leads to collectively worst-case outcomes~\citep{lola}. To address this, \textit{opponent shaping} (OS) methods account for opponents’ learning steps and influence other agents' learning processes. Empirically, this can improve individual and group performances.


Early OS methods~\citep{lola, letcher2018stable, meta-MAPG} rely on higher-order derivatives, which are high-variance and result in unstable learning. They are also myopic, focusing only on the opponent's immediate future learning steps rather than their long-term development \citep{lu2022mfos}. Recent work, Model-free Opponent Shaping (M-FOS) \citep{lu2022mfos}, solves the above challenges. M-FOS introduces a \textit{meta-game} structure, each \textit{meta-step} representing an episode of the embedded ``inner'' game. The \textit{meta-state} consists of ``inner'' policies, and the \textit{meta-policy} generates an inner policy at each \textit{meta-step}. M-FOS uses model-free optimisation techniques to train the meta-policy, eliminating the need for higher-order derivatives to accomplish long-horizon opponent shaping. The M-FOS framework has shown promising long-term shaping results in social-dilemma games \citep{lu2022mfos,khancontext}.

The original M-FOS paper presents two cases of the M-FOS algorithm. For simpler, low-dimensional games, M-FOS learns policy updates directly by taking policies as input and outputting the next policy as an action. Inputting and outputting entire policies does not extend well to more complex, higher-dimensional games, e.g. when policies are represented as neural networks. The original M-FOS paper also proposes a variant which uses trajectories as inputs instead of the exact policy representations. In this work we derive the sample complexity for both cases.

Whereas some previous OS algorithms enjoy strong theoretical foundations thanks to the Differentiable Games framework \citep{balduzzi2018mechanics}, the M-FOS framework has not been investigated theoretically. Understanding the sample complexity of an algorithm is helpful in many ways, such as evaluating its efficiency or predicting the learning time. However, providing theoretical guarantees for M-FOS is challenging because A) there is very little literature on theoretical sample complexity bounds for even single-agent meta-reinforcement learning (RL), let alone multi-agent and B) M-FOS operates in continuous state and action space (the meta-game). 


In this work, we present R-FOS, a tabular algorithm approximation of M-FOS. Unlike M-FOS, which operates in a continuous meta-MDP, R-FOS operates in a discrete approximation of the original meta-MDP. The resulting \textit{discrete} MDP allows us to perform rigorous theoretical analysis. We adapt R-FOS from M-FOS such that it still maintains all the key properties of M-FOS. 
Within this discrete, approximate MDP, R-FOS applies the $R_\text{MAX}$ algorithm \cite{pacmdp, kakade_2003} to the M-FOS meta-game. $R_\text{MAX}$ is a model-based reinforcement learning (MBRL) algorithm typically used for the sample complexity analysis of tabular MDPs. Using existing results developed for $R_\text{MAX}$ \citep{rmax}, we derive an exponential sample complexity PAC-bound, which guarantees with high probability (1-$\delta$) that the optimal policy in the discretised meta-MDP is very close ($<\epsilon$ away) to the policy learned by R-FOS. We then derive several bounds which guarantee policies between the original meta-MDP and the discretised meta-MDP are close to each other up to a constant distance. Lastly, combining all of the previous bounds we derived, we obtain the final exponential sample complexity result.

For notational simplicity, we mostly omit the ``meta'' prefix in the rest of the paper. For example, the terms ``MDP'', ``transition function'', and ``policy'', refer to the meta-MDP, meta-transition function, and meta-policy respectively. We use the prefix ``inner'' whenever we refer to the inner game.
Furthermore, our analysis of M-FOS is limited to the asymmetric \textit{shaping} case (i.e. the meta-game of shaping a naive inner-learner\footnote{Naive learners are players who update their policy assuming other learning agents are simply a part of the environment.}) and we leave the extension to meta-selfplay for future work.

Our contributions are three-fold:

\begin{enumerate} [leftmargin=0.5cm]
    \item We present R-FOS (see Algorithm \ref{alg:rmax}), a tabular approximation of M-FOS. Instead of learning a meta-policy inside the continuous meta-MDP $M$, R-FOS learns a meta-policy inside a discretised meta-MDP which approximates $M$. Inside this discretised Meta-MDP, R-FOS uses $R_\text{MAX}$ as the meta-agent. Note that R-FOS still maintains key properties of the original M-FOS algorithm, such as being able to exploit naive learners.
    \item We present an exponential sample complexity bound for both cases described in M-FOS (See Theorems \ref{thm:pac1} and \ref{thm:pac2}). Specifically, we prove that, with high probability, the final R-FOS policy is close to the optimal policy in the original meta-MDP up to a constant distance.
    \item We implement R-FOS \footnote{The project code is available on https://github.com/FLAIROx/rfos} and analyse the empirical sample complexity in the \textit{Matching Pennies} environment. We establish links between theory and experiments by demonstrating that in both realms, sample complexity scales exponentially with the inner-game's state-action space size.

\end{enumerate}

\section{Related Work}
\label{section:relatedwork}

\textbf{Theoretical Analysis of Differentiable Games:} Much past work assumes that the game being optimised is differentiable \cite{balduzzi2018mechanics}. This assumption enables far easier theoretical analysis because one can directly use end-to-end gradient-based methods rather than reinforcement learning in those settings. Several works in this area investigate the convergence properties of various algorithms \citet{letcher2020impossibility,schafer2019competitive, balduzzi2018mechanics}.

\textbf{Opponent Shaping:} More closely related to our work are methods that specifically analyse OS. SOS \cite{letcher2018stable} and COLA \cite{willi2022cola} both analyse opponent-shaping methods that operate in the differentiable games framework. These works provide theoretical \textit{convergence} analysis for opponent-shaping algorithms; however, neither work analyzes sample complexity. POLA \cite{zhao2022proximal} theoretically analyses an OS method that is invariant to policy parameterization. M-FOS does not operate in the differentiable games framework. While this enables M-FOS to scale to more challenging environments, such as Coin Game \cite{lu2022mfos}, it comes at the cost of convenient theoretical analysis. \citet{khancontext} empirically scales M-FOS to more challenging environments with larger state spaces, while \citet{lu2022adversarial} empirically investigates applying M-FOS to a state-based adversary. To the best of our knowledge, our work is the first to theoretically analyse OS outside of the differentiable games framework. Furthermore, our work is the first to analyse the sample complexity of an OS method.

\textbf{Theoretical Analysis of Sample Complexity in RL:} There are several works that use the $R_\text{MAX}$ \cite{rmax} framework to derive the sample complexity of RL algorithms across a variety of settings. Closely related to our work is \citet{marco}, which uses the $R_\text{MAX}$ algorithm to derive sample complexity bounds for learning in fully-cooperative multi-agent RL. 

Our work is also related to methods that analyse sample complexity on continuous-space RL. Analyzing the sample complexity of algorithms in continuous-space RL is particularly challenging because there are an infinite number of potential states. To address this, numerous techniques have been suggested that each make specific assumptions:
\citet{liu2018} assumes a \textbf{stationary asymptotic occupancy distribution} under a random walk in the MDP. \citet{Malik2021SampleER} uses an effective planning window to handle MDPs with non-linear transitions. However, neither of these assumptions applies to M-FOS.

Instead, this work focuses on \textbf{discretising} the continuous space and expresses the complexity bounds in terms of the discretisation grid size. This is related to the concept of \textit{state aggregation} \cite{softstate, aggregation}, which groups states into clusters and treats the clusters as the states of a new MDP. These previous works only formulated the aggregation setting in MDPs and did not provide theoretical or empirical sample complexity proofs.

Furthermore, prior studies on \textit{PAC-MDP} did not empirically verify the connection between the sample complexity and size of the state-action space. In this work, we \textbf{empirically verify the relationship between the sample complexity and the cardinality of the inner-state-action-space} in the Matching Pennies game.

\begin{algorithm}[h!]
    \caption{The R-FOS Algorithm}
    \label{alg:rmax}
    \textbf{Meta-game Inputs:} Discretised meta-MDP $\langle \hat{S}_d, \hat{A}_d, T_d, R_d, \gamma \rangle$ $m$-known, meta-game horizon $h_\text{meta}$ \\
    \textbf{Inner-game Inputs:} Inner game G = $\langle S, A, T_\text{inner}, R_\text{inner}\rangle$ inner-game horizon $h_\text{inner}$\\
    \textbf{Initialisation:} $\forall \hat{s}_d \in \hat{S}_d, \hat{a}_d \in \hat{S}_d, \hat{s}_d^\prime \in \hat{S}_d$ $\hat{Q}(\hat{s}_d, \hat{a}_d)\gets 0,  \: r(\hat{s}_d, \hat{a}_d) \gets 0 , \:  n(\hat{s}_d, \hat{a}_d)\gets 0, \:  n(\hat{s}_d, \hat{a}, \hat{s}')\gets 0$

        \begin{algorithmic}[1]
        \FOR{meta-episode $ = 0,1,..$}    
            \STATE Reset environment    
            \FOR{meta-time step $= 1, 2, ..., h_\text{meta}$} 
                \STATE Choose $\hat{a}_{d} := \argmax_{\hat{a}_d \in \hat{A}_d}\hat{Q}(\hat{s}_d, \hat{a}_d^\prime)$
                \STATE Roll-out $K$ inner games of length $h_\text{inner}$ using 
                $\hat{a}_{d} = \phi_t$
                \STATE Inner-game opponents each update their own inner-policies naively
                \STATE Let $R$ be our agent's $K$ inner-games' discounted return
                \STATE Let $\hat{s}_d^\prime$ be the next meta-state after executing meta-action $\hat{a}_d$ from meta-state $\hat{s}_d$
               
                \IF{$n(\hat{s}_d, \hat{a}_d)<m$}
                    \STATE $r(\hat{s}_d, \hat{a}_d) \gets r(\hat{s}_d, \hat{a}_t)+R$
                    \STATE $n(\hat{s}_d, \hat{a}_d) \gets n(\hat{s}_d, \hat{a}_t)+1$   
                    \STATE $n(\hat{s}_d, \hat{a}_d, \hat{s'}_d) \gets n(\hat{s}_d, \hat{a}_d, \hat{s}^\prime_d)+1$
                    \IF{$n(\hat{s}_d, \hat{a}_d)=m$}
                        \FOR{$i=1,2,3,\cdots,\lceil \frac{ln(\frac{1}{\varepsilon(1-\gamma)})}{1-\gamma} \rceil$}
                            \FORALL{$(\hat{s},\hat{a})$}
                                \IF{$n(\hat{s}_d,\hat{a}_d) \geq m$}
                                    \STATE $\hat{Q}(\hat{s}_d,\hat{a}_d) \gets \hat{R}_d(\hat{s}_d, \hat{a}_d) + \gamma \sum_{s^\prime_d}\hat{T}_d(\hat{s}^\prime|\hat{s}_d, \hat{a}_d) \max_{\hat{a}_d^\prime}\hat{Q}(\hat{s}_d^\prime, \hat{a}_d^\prime)$
                                \ENDIF
                            \ENDFOR
                        \ENDFOR
                    \ENDIF
                \ENDIF
                \STATE $\hat{s} \gets \hat{s}'$            
            \ENDFOR
        \ENDFOR
    \end{algorithmic}
\end{algorithm}

\section{Background}
\label{section:Background}

\subsection{Stochastic Game}
A stochastic game (SG)\footnote{We use the \textbf{bold} notation to indicate vectors over $n$ agents.} is given by a tuple $G = \langle\mathcal{I},  S, \bm{ A}, T, \bm{R}, \gamma\rangle$. 
$\mathcal{I} = \{1, \cdots, n\}$ is the set of agents, $ S$ is the state space, $\bm{ A}$ is the cross-product of the action space for each agent such that the joint action space $\bm{ A} =  A^1 \times \cdots \times  A^n$, 
$T:  S \times \bm{ A} \mapsto  S$ is the transition function, $\bm{R}$ is the cross-product of reward functions for all agents such that the joint reward space $\bm{R} = R^1 \times \cdots \times R^n$, and $\gamma \in [0,1)$ is the discount factor.

In an SG, agents simultaneously choose an action according to their stochastic policy at each timestep $t$, $a^i_t \sim \pi^i_{\phi^i}(\cdot|s^i_t)$. The joint action at timestep $t$ is $\bm{a_t} = \{a^i_t, \bm{a^{-i}_t}\}$, where the superscript $\bm{-i}$ indicates all agents except agent $i$ and $\phi^i$ is the policy parameter of agent $i$. The agents then receive reward $r^i_t = R^i(s_t, \bm{a_t})$ and observe the next state $s_{t+1} \sim T(\cdot | s_t, \bm{a_t})$, resulting in a trajectory $\tau^{i} = (s_0, \bm{a_0}, r^i_0, ..., s_T, \bm{a_T}, r^i_T)$, where $T$ is the episode length.

\subsection{Markov Decision Process}
A Markov decision process (MDP) is a special case of stochastic game and can be described as $\mathcal{M}=$ $\langle S,  A, T, R, \gamma\rangle$, where $ S$ is the state space, $ A$ is the action space, $T\left(s_{t+1} \mid s_t, a_t\right)$ is the transition function, $R\left(s_t, a_t\right)$ is the reward function, and $\gamma$ is the discount factor. At each timestep $t$, the agent takes an action $a_t \in  A$ from a state $s_t \in  S$ and moves to a next state $s_{t+1} \sim T\left(\cdot \mid s_t, a_t\right)$. Then, the agent receives a reward $r_t=R\left(s_t, a_t\right)$.

\subsection{Model-Free Opponent Shaping}
Model-free Opponent Shaping (M-FOS) \cite{lu2022mfos}
frames the OS problem as a meta-reinforcement-learning problem, in which the opponent \textit{shaper} plays a meta-game. The meta-game is an MDP (sometimes we also refer to it as meta-MDP ), in which the meta-agent controls one of the inner agents in the inner game. 

The inner game is the actual environment that our agents are playing, which is an SG. The original M-FOS describes two cases for the meta-state:
\begin{enumerate}
    \item In the meta-game at timestep $t$, the M-FOS agent is at the meta-state $\hat{s}_t = [\phi^i_{t-1}, \bm{\phi^{-i}_{t-1}}]$, which contains all inner-agents' policy parameters for the underlying SG. In this work, we assu,e all inner-agents are parameterised by their  Q-value table.
    \item Alternatively, $\hat{s}_t = \bm{\tau}$ in cases where past trajectories of the inner-game represent the policies sufficiently.
\end{enumerate}
We provide theoretical sample complexity results for both of these two cases

The meta-agent takes a meta-action $\hat{a}_t = \phi^i_t \sim \pi_\theta(\cdot|\hat{s}_t)$, which is the M-FOS' inner agent's policy parameters. The action is chosen from the meta-policy $\pi$ parameterized by parameter $\theta$. In this work, we only look at the case where the meta-policy is a Q-value function table, and is denoted as $\hat{Q}$ instead. The M-FOS agent receives reward $r_t = \sum_{k=0}^K r^i_k(\phi^i_t, \bm{\phi^{-i}_t})$, where $K$ is the number of inner episodes. A new meta-state is sampled from a stochastic transition function $\hat{s}_{t+1} \sim T(\cdot| \hat{s}_t, \hat{a}_t)$. 

Note that the original paper introduces two different algorithms: The first meta-trains M-FOS against \textit{naive learners} commonly resulting in exploiting them. The second instead considers \textit{meta-self-play}, whereby two M-FOS agents are trained to shape each other, resulting in reciprocity. In this work we only consider the first, asymmetric case. 

\subsection{The $R_\text{MAX}$ Algorithm}

$R_\text{MAX}$ \citep{rmax} is an MBRL algorithm proposed for analysing the sample complexity for tabular MDPs. Given any MDP $M$, $R_\text{MAX}$ constructs an \textit{empirical MDP} $\hat{M}$ that approximates $M$. The approximation is done by estimating the reward function $R$ and transition $T$ using \textit{empirical} samples. The resulting approximate reward and transition models are denoted by $\hat{R}$ and $\hat{T}$ respectively. 

$R_\text{MAX}$ encourages exploration by dividing the state-action pairs into two groups - those that have been visited at least $m$-times, and those that haven't. The set of state-action pairs that have been visited at least $m$-times is called the ``$m$-known set''. Using the empirical MDP $\hat{M}$, the $R_\text{MAX}$ algorithm constructs an \textit{$m$-known empirical MDP}. This $m$-known empirical MDP behaves almost exactly as the empirical MDP, except when the agent is at a state-action pair outside the $m$-known set. When the agent is outside the $m$-known set, the transition function is self-absorbing (i.e. the transition function only transitions back to the current state) and the reward function is the maximum (See Table \ref{tb:theMs_appendix} in the appendix). The consequence of the $m$-known setup is the agent is encouraged to explore state-action pairs that have high uncertainty (i.e. that has been visited under $m$ times). Specifically, the value function for the under-visited states is the maximum possible expected return, which gives the algorithm its name. This is in line with \textit{optimism in the face of uncertainty}.

\subsection{$\varepsilon$-Nets}

\begin{definition}($\varepsilon$-Net \cite{net, 10.1145/10515.10522}) For $\varepsilon>0$, $\mathcal{N}_\varepsilon$ is an $\varepsilon$-net over the set $\Theta \subseteq \mathbb{R}^D$ if for all $\theta \in \Theta$, there exists $\theta' \in \mathcal{N}_\varepsilon$ such that $\norm{\theta -\theta'}_2 \leq \varepsilon$.
\end{definition}

To discretise a $D$-dimensional sphere of radius $R$, we can use a $\varepsilon$-net containing $D$-dimensional cubes of sides $\lambda$. This results in $\left(\frac{2R}{\lambda} + 1\right)^D$ points. Within each $D$-dimensional cube, the largest distance between the vertices and the interior points comes from the center of the cube, which is $\frac{\lambda\sqrt{d}}{2}$. Therefore, to guarantee a full cover of all the points in the sphere, the largest cube size that we can have should satisfy $\varepsilon = \frac{\lambda\sqrt{d}}{2}$.From here on, we will replace the $\varepsilon$ in $\varepsilon$-net with $\alpha$ to avoid notation overloading.

\section{Sample Complexity Analysis with $R_\text{MAX}$ as Meta-Agent} \label{section:5}
As introduced in Section~\ref{section:Background}, $R_\text{MAX}$ \citep{rmax} is a MBRL algorithm for learning in tabular MDPs. We adapt the original M-FOS algorithm to use $R_\text{MAX}$ as the meta-agent (see Algorithm \ref{alg:rmax}) and refer to this adapted algorithm as \textit{R-FOS} from here on. We use a \textit{tabular Q-learner} as the naive learner for all inner-game opponents. While the original M-FOS paper uses PPO \citep{schulman2017proximal}, we choose the Q-learner for the ease of sample complexity analysis.

We provide theoretical results for the two cases of M-FOS' meta-agent proposed by the original paper \cite{lu2022mfos}. \textit{Case I} uses all agents' inner policy parameters from the previous timestep as the meta-state. \textit{Case II} instead uses the most recent inner-game trajectories as the meta-state. In both cases, the meta-action determines the inner agent's policy parameters for the next inner episode.

At a high level, we first discretise the meta-MDP, which allows us to use the theoretical bounds from $R_\text{MAX}$ (only suitable for tabular MDPs), then we develop theory for bounding the discrepancy between the continuous and discrete meta-MDP, and lastly, we use all of this to bound the final discrepancy. 
Specifically, the sample complexity analysis consists of six steps \footnote{see detailed proof in the appendix}:

\begin{enumerate}
    \item To use $R_\text{MAX}$ as the M-FOS meta-agent, we first discretise the continuous meta-MDP $M = \langle\hat{S},\hat{A}, T, R, \gamma\rangle$ into a discretised meta-MDP $M_d=\langle\hat{S}_d,\hat{A}_d, T_d, R_d, \gamma \rangle$. We first discretise the continuous meta-state space and meta-action space using epsilon-nets \cite{net, 10.1145/10515.10522}. Based on this discretised meta-state and meta-action space, we define the discretised transition and reward function. See Section \ref{step1} for details.
    \item  We then construct a $m$-known discretised MDP $M_m,d$, as described by the R-MAX algorithm \cite{pacmdp}. See Section \ref{step2} for details.
    \item Then, we deploy the $R_\text{MAX}$ algorithm in $M_m,d$. $R_\text{MAX}$ both estimates the empirical $m$-known discretised MDP, $\hat{M}_m,d$, using a maximum likelihood estimate from empirical samples and learns an optimal policy in $\hat{M}_m,d$. For example, to estimate the meta-reward, our algorithm, R-FOS evaluates the inner-game policy outputted by the meta-policy using episodic rollouts. The estimates are then used to update the meta-policy according to the R-FOS algorithm. Our R-FOS algorithm \textit{optimistically} assigns rewards for all under-visited discretised (meta-state, meta-action) pairs to encourage exploration (like $R_\text{MAX}$). See Section \ref{step3} for details.
    
    \item We next prove a PAC-bound which guarantees with large probability, that the optimal policy learnt in $\hat{M}_m,d$ is similar to the optimal policy in $M,d$. This step uses results from \citep{pacmdp}. See Section \ref{step4} for details.
    
    \item We also prove a strict bound that guarantees the optimal policies learnt in $M,d$ and $M$ are similar up to a constant. This step uses results from \citep{chow1991optimal}. See Section \ref{step5} for details.
    
    \item Using the two bounds from above, we prove the final sample complexity guarantee which quantifies that, with large probability, the optimal policy learnt in $\hat{M}_m,d$ is similar to the optimal policy in $M$ up to a constant. See Section \ref{step6} for details.
\end{enumerate}

\subsection{\textbf{Assumptions}}
We first outline all assumptions made in deriving the sample complexity of the R-FOS algorithm.
\begin{assumption} \label{ass:timestep}
Both meta-game and inner-game are finite horizon. We use $h_\text{meta}$ to denote the meta-game horizon, and $h_{\text{inner}}$ to denote the inner-game horizon.
\end{assumption}
\begin{assumption}
\label{ass:bounded_reward}
We assume the inner-game reward is bounded. For simplicity of the proof and without loss of generality, we set this bound as $\frac{1}{h_\text{inner}}$, where $h_\text{inner}$ is the horizon of the inner game. Formally, for all $(s, a), 0 \leq R_\text{inner}(s, a) \leq \frac{1}{h_\text{inner}}$. This allows us to introduce the notion of \textit{maximum inner reward} and \textit{maximum inner value function} as $R_{\max, \text{inner}} = \frac{1}{h_\text{inner}}$ and $V_{\max, \text{inner}} = 1$ respectively. This implies that the reward and value function in the meta-game are also bounded, i.e., $R_{\max} = 1$ and $V_{\max} = \frac{1}{1-\gamma}$ (the latter being an upper bound).
\end{assumption}
\begin{assumption} \label{ass:gamma}
The meta-game uses a discount factor of $\gamma$. For simplicity of the proof, the inner-game uses a discount factor of $1$. This assumption can be easily deleted by adapting $R_\text{max, inner}$ (see above) in the original proof in \citep{pacmdp}.
\end{assumption}
\begin{assumption} \label{ass:pomdp}
For simplicity, the inner game is assumed to be discrete.
\end{assumption}
\begin{assumption} \label{ass:lips_rew}
The meta-reward function is Lipschitz-continuous: For all $ \hat{s}_1, \hat{s}_2\in \hat{S} \text{ and } \hat{a}_1, \hat{a}_2 \in \hat{A},$
$$
\abs{R(\hat{s}_1, \hat{a}_1) - R(\hat{s}_2, \hat{a}_2)} \leq \mathcal{L}_R \norm{(\hat{s}_1, \hat{a}_1) - (\hat{s}_2, \hat{a}_2)}_\infty
$$
where $\mathcal{L}_R$ is the meta-reward function's Lipschitz-constant.
\end{assumption}

\begin{assumption} \label{ass:lips_trans}
The meta-transition function is Lipschitz-continuous: For all $ \hat{s}_1, \hat{s}_1^\prime, \hat{s}_2\in \hat{S} \text{ and } \hat{a}_1, \hat{a}_1^\prime, \hat{a}_2 \in \hat{A},$
$$
\abs{T(\hat{s}_1^{\prime} \mid \hat{s}_1, \hat{a}_1) - T(\hat{s}_2^{\prime} \mid \hat{s}_2, \hat{a}_2)} \leq \mathcal{L}_T \norm{(\hat{s}_1^\prime, \hat{s}_1, \hat{a}_1) - (\hat{s}_2^\prime, \hat{s}_2, \hat{a}_2)}_\infty
$$
where $\mathcal{L}_T$ is the meta-transition function's Lipschitz-constant.
\end{assumption}

\begin{assumption} \label{ass:point_to_set}
There's a Lipschitz-continuous point-to-set mapping between meta-state space and meta-action space such that for any $\hat{s}, \hat{s}^\prime \in \hat{S}$ and $\hat{a}, \hat{a}^\prime \in \hat{A}$, there exists some $\hat{a} \in U(\hat{s})$ such that $\norm{\hat{a} - \hat{a}^\prime}_\infty < \mathcal{L} \norm{\hat{s} - \hat{s}^\prime}_\infty$
\end{assumption}

\begin{assumption} \label{ass:pdf}
The meta-game transition function $T(\cdot \mid \hat{s}, \hat{a})$ is a probability density function such that $0 \leq T(\hat{s}^\prime \mid \hat{s}, \hat{a}) \leq \mathcal{L} \text{ and } \int_{\hat{S}} T(\hat{s}^\prime \mid \hat{s}, \hat{a}) \, d\hat{s}^\prime = 1, \quad \forall \hat{s}, \hat{s}^\prime \in \hat{S} \text{ and } \hat{a} \in \hat{A}$
\end{assumption}

The first four assumptions are required to be able to use the R-MAX algorithm, while the latter assumptions are needed for bounding the discrepancy between the continuous meta-MDP and the discretised meta-MDP. 

\subsection{Step 1: Discretising the Meta-MDP}
\label{step1}
 To use $R_\text{MAX}$ as the M-FOS meta-agent, we discretise the continuous meta-MDP $M = \langle\hat{S},\hat{A}, T, R, \gamma\rangle$ into a discretised meta-MDP $M_d=\langle\hat{S}_d,\hat{A}_d, T_d, R_d, \gamma \rangle$. We discretise the continuous state and action space using $\varepsilon$-nets with spacing $\alpha$.


\subsubsection{Discretising the State and Action Space: Case I} \label{ssection:arch1}

In \textit{Case I}, the meta-state $\hat{s}_{t}$ is all inner agents' policies parameters from the previous timestep. Each of the inner agent $i$'s policy is a Q-table, denoted as $\phi_i \in \mathbb{R}^{|S| \times |A|}$.
Formally, $\hat{s}_{t} := \bm{\phi}_{t-1} = [\phi_{t-1}^{i}, \bm{\phi_{t-1}^{-i}}].$ The meta-action $\hat{a}_t$ is the inner agent's current policy parameters $\phi^{i}_t$. 

For the meta-action space $\hat{A}$ and a chosen discretisation error $\alpha>0$, we obtain the $\varepsilon$-net $\hat{A}_d \subset \hat{A}$ such that for all $\hat{a} \in \hat{A}$, there exist $\hat{a}_d \in \hat{A}_d$ where 
\begin{equation}
\label{eq:alpha}
    \norm{\hat{a} - \hat{a}_d} \leq \alpha.
\end{equation}

Dividing the space with grid size $\lambda$ results in the size of discretised meta-action space upper bounded by
\begin{equation}
    \label{eq:metaactionsize}
    |\hat{A}_d| \leq \left(\frac{2 \sqrt{|S||A|}} {\lambda} + 1 \right)^{|S||A|}.
\end{equation}

Similarly, the size of the discretised meta-state space is upper bounded by
\begin{equation}
    \label{eq:metastatesize1}
    |\hat{S}_d| \leq \left(\frac{2 \sqrt{n|S||A|}} {\lambda} + 1 \right)^{n|S||A|}.
\end{equation}

\subsubsection{Discretising the State and Action Space: Case II} \label{ssection:arch1}

In Case II, the meta-state $\hat{s}_t$ is all inner agents' past trajectories. Formally, $\hat{s}_{t} := \bm{\tau}_t.$, where $\tau^i_t = \{s_0, a_0, s_1, a_1, ..., s_t, a_t\}$. Because we assume the Inner-Game is discrete (i.e. the state and action space are both discrete), the meta-state in this case does not need discretisation. Let $h$ be the maximum length of the past trajectories combined, i.e. $h=h_\text{inner} \cdot h_\text{meta}$. The size of the meta-state space is
\begin{equation}
\label{eq:metastatesize2}
    |\hat{S}_t| = (|S||A|)^{nh}.
\end{equation}
The meta-action remains the same as Case I.
\subsubsection{Discretising the Transition and Reward Function}
Under the above discretisation procedure, we define the discretised MDP $M_d=(\hat{S}, \hat{A}_d, T_d, R_d, \gamma)$, where the state space remains continuous and the action space is restricted to \textit{discretised} actions. We define the transition function and reward function for $M_d$ as:
\begin{equation}
    \label{eq:discrete_transition}
    T_d(\hat{s}'\mid\hat{s},\hat{a}_d) = \frac{T(\hat{s}'_d \mid \hat{s}_d, \hat{a}_d)}{\int_{\hat{S}} T(\hat{z}_d\mid \hat{s}_d, \hat{a}_d)d\hat{z}}
\end{equation}
Intuitively, $T_d(\hat{s}'\mid \hat{s}_d,\hat{a}_d)$ is a normalized sample of $T_d(\cdot\mid \cdot,\hat{a}_d)$ at $\hat{s}',\hat{s}_d$, and the transition probability $T_d(\cdot\mid \hat{s}, \hat{a}_d)$ takes a constant value within each grid in the state space. This means that instead of treating the transition function as a discretised distribution of all possible values of $\hat{S}_d$, we treat it as a continuous distribution over the original continuous state space, but normalize each grid from the $\varepsilon$-net into a step function.
\begin{equation}
    \label{eq:discrete_reward}
    R_d(\hat{s},\hat{a}_d) = R(\hat{s}_d, \hat{a}_d)
\end{equation}
Similarly, the reward function is continuous over the state space, but normalized each grid from the $\varepsilon$-net into a step function.

\subsection{Step 2: The $m$-known Discretised MDP}
\label{step2}


In the previous step, we converted the meta-MDP $M$ into a \textit{discretised} meta-MDP $M_d$. From $M_d$, R-FOS builds an $m$-known discretised MDP $M_{m,d}$ (see Table \ref{tb:theMs_appendix} in the appendix). 

\begin{definition}[m-Known MDP]
\label{def:known}
Let $M_d=\langle\hat{S}_d, \hat{A}_d, T_d, R_d, \gamma\rangle$ be an MDP. We define $M_m,d$ to be the $m$-known MDP. As is standard practice, $m$-known refers to the set of state-action pairs that have been visited at least $m$ times. For all state-action pairs in $m$-known, the induced MDP $M_{m,d}$ behaves identical to $M_d$. For state-action pairs outside of $m$-known, the state-action pairs are self-absorbing (i.e. only self-transitions) and maximally rewarding with $RMAX$. 

\end{definition}

\subsection{Step 3: The Empirical Discretised MDP}
\label{step3}
From the $m$-known discretised MDP $M_{m,d}$, we then learn an \textit{empirical} $m$-known discretised MDP $M_{m,d}$ by calculating the maximum likelihood from empirical samples (see Table \ref{tb:theMs_appendix} in the appendix). As shown in Algorithm \ref{alg:rmax}, R-FOS learns an optimal policy within this empirical $m$-known discretised MDP.

\begin{definition}[Empirical m-Known discretised MDP]
\label{def:known}
$M_{m,d}$ is the expected version of $\hat{M}_{m,d}$ where:
\vspace{-5pt}
\begin{center}
\begin{equation}
\begin{aligned}
    T_{m,d}(\hat{s}_d^{\prime} \mid \hat{s}_d, \hat{a}_d) & :=\left\{\begin{array}{lll} T_d(\hat{s}_d^{\prime} \mid \hat{s}_d, \hat{a}_d) & \text { if } (\hat{s}_d, \hat{a}_d) \in \text{m-known}\hfill \\ \mathbbm{1}[\hat{s}_d^{\prime}=\hat{s}_d], & \text { otherwise}\hfill & \end{array}\right.\\
    \hat{T}_{m,d}(\hat{s}_d^{\prime} \mid \hat{s}_d, \hat{a}_d) & := \begin{cases} \frac{n(\hat{s}_d, \hat{a}_d, \hat{s}_d^{\prime})}{n(\hat{s}_d, \hat{a}_d)}, & \text { if }(\hat{s}_d, \hat{a}_d) \in \text{m-known}\hfill \\ \mathbbm{1}[\hat{s}_d^{\prime}=\hat{s}_d], & \text { otherwise }\hfill\end{cases} \\
     R_{m,d}(\hat{s}_d, \hat{a}_d)& :=\left\{\begin{array}{lll} R_d(\hat{s}_d, \hat{a}_d), & \text { if } (\hat{s}_d, \hat{a}_d) \in \text{m-known}\hfill \\ R_{\max} & \text { otherwise}\hfill\end{array}\right. \\
     \hat{R}_{m,d}(\hat{s}_d, \hat{a}_d)&= \begin{cases}\frac{\sum_{i}^{n(\hat{s}_d, \hat{a}_d)}r(\hat{s}_d, \hat{a}_d)}{n(\hat{s}_d, \hat{a}_d)}, & \text { if }(\hat{s}_d, \hat{a}_d) \in \text{m-known}\hfill \\ R_{\max}, & \text { otherwise }\hfill\end{cases} \\
\end{aligned}
\end{equation}
\end{center}
\end{definition}

\subsection{Step 4: The Bound Between $M_d$ and $\hat{M}_{m,d}$}
\label{step4}
We first prove the PAC bound which guarantees that, with high probability, the optimal policies learnt in the discretised MDP $M_d$ and empirical $m-known$ discretised MDP are very close. We prove the bound using results from \cite{pacmdp}.

\begin{theorem} ($R_\text{MAX}$ MDP Bound \cite{pacmdp})
\label{eq:strehl}
Suppose that $0 \leq \varepsilon<\frac{1}{1-\gamma}$ and $0 \leq \delta<1$ are two real numbers and $M=\langle S, A, T, R, \gamma\rangle$ is any MDP. There exists inputs $m=m\left(\frac{1}{\varepsilon}, \frac{1}{\delta}\right)$ and $\varepsilon_1$, satisfying $m\left(\frac{1}{\varepsilon}, \frac{1}{\delta}\right)=O\left(\frac{(S+\ln (S A / \delta)) V_{\max }^2}{\varepsilon^2(1-\gamma)^2}\right)$ and $\frac{1}{\varepsilon_1}=O\left(\frac{1}{\varepsilon}\right)$, such that if $R_\text{MAX}$ is executed on $M$ with inputs $m$ and $\varepsilon_1$, the following holds. Let $A_t$ denote $R_\text{MAX}$'s policy at time $t$ and $s_t$ denote the state at time $t$. With probability at least $1-\delta$, $V_M^{A_t}\left(s_t\right) \geq V_M^*\left(s_t\right)-\varepsilon$ is true for all but
$$
\tilde{O}\left(S^2 A /\left(\varepsilon^3(1-\gamma)^6\right)\right)
$$
timesteps (final sample complexity bound).

\end{theorem}

\subsection{Case I}

Directly plugging in Equations $\ref{eq:metastatesize1}$ and $\ref{eq:metaactionsize}$ into Theorem \ref{eq:strehl}, we obtain the following PAC-bound.

\begin{theorem} 
\label{thm:pac1}
Suppose that $0 \leq \varepsilon<\frac{1}{1-\gamma}$ and $0 \leq \delta<1$ are two real numbers. Let M be any continuous meta-MDP with inner stochastic game $G = \langle S, A, T_\text{inner}, R_\text{inner} \rangle$.
Let us denote $M_d=\langle \hat{S}_d, \hat{A}_d, T_d, R_d, \gamma\rangle$ as discretised version of $M$ (as described in Case I) using grid size of $\lambda$. There exists inputs $m=m\left(\frac{1}{\varepsilon}, \frac{1}{\delta}\right)$ and $\varepsilon_1$, satisfying 
$$m\left(\frac{1}{\varepsilon}, \frac{1}{\delta}\right)=\tilde{O}\left(\frac{\left(\frac{2 \sqrt{n|S||A|}} {\lambda} + 1 \right)^{n|S||A|}}{\varepsilon^2(1-\gamma)^4}\right)$$
and $\frac{1}{\varepsilon_1}=O\left(\frac{1}{\varepsilon}\right)$, such that if $R\text{-}MAX$ is executed on $M$ with inputs $m$ and $\varepsilon_1$, then the following holds. Let $\mathcal{A}_t$ denote $R\text{-}MAX$'s policy at time $t$ and $\hat{s}_t$ denote the state at time $t$. With probability at least $1-\delta$, $ V_{M_d}^*\left(\hat{s}_t\right) - V_{M_d}^{\mathcal{A}_t}\left(\hat{s}_t\right) \leq \varepsilon$ is true for all but
$$
\tilde{O}\left(\frac{\left(\frac{2 \sqrt{n|S||A|}} {\lambda} + 1 \right)^{2n|S||A|} \left(\frac{2 \sqrt{|S||A|}} {\lambda} + 1 \right)^{|S||A|}} {\varepsilon^3(1-\gamma)^6}\right)
$$
timesteps.

\end{theorem}

\subsection{Case II}
Directly plugging in Equations $\ref{eq:metastatesize2}$ and $\ref{eq:metaactionsize}$ into Theorem \ref{eq:strehl}, we obtain the folling PAC-bound.

\begin{theorem} 
\label{thm:pac2}
Suppose that $0 \leq \varepsilon<\frac{1}{1-\gamma}$ and $0 \leq \delta<1$ are two real numbers. Let M be any continuous meta-MDP with inner stochastic game $G = \langle S, A, T_\text{inner}, R_\text{inner} \rangle$.
Let us denote $M_d=\langle \hat{S}_d, \hat{A}_d, T_d, R_d, \gamma\rangle$ as discretised version of $M$ (as described in Case I) using grid size of $\lambda$. There exists inputs $m=m\left(\frac{1}{\varepsilon}, \frac{1}{\delta}\right)$ and $\varepsilon_1$, satisfying 
$$m\left(\frac{1}{\varepsilon}, \frac{1}{\delta}\right)=\tilde{O}\left(\frac{(|S||A|)^{nh}}{\varepsilon^2(1-\gamma)^4}\right)$$
and $\frac{1}{\varepsilon_1}=O\left(\frac{1}{\varepsilon}\right)$, such that if $R\text{-}MAX$ is executed on $M$ with inputs $m$ and $\varepsilon_1$, then the following holds. Let $\mathcal{A}_t$ denote $R\text{-}MAX$'s policy at time $t$ and $\hat{s}_t$ denote the state at time $t$. With probability at least $1-\delta$, $ V_{M_d}^*\left(\hat{s}_t\right) - V_{M_d}^{\mathcal{A}_t}\left(\hat{s}_t\right) \leq \varepsilon$ is true for all but
$$
\tilde{O}\left(\frac{(|S||A|)^{nh} \left(\frac{2 \sqrt{|S||A|}} {\lambda} + 1 \right)^{|S||A|}} {\varepsilon^3(1-\gamma)^6}\right)
$$
timesteps.
\end{theorem}

\subsection{Step 5: The Bound between $M$ and $M_d$}
\label{step5}
Next, we give a guarantee that the optimal policies learnt in the original meta-MDP $M$ and the discretised MDP $M_d$ are similar enough with a distance up to a constant factor. Using the results from \cite{chow1991optimal}, we obtain the following property.

\begin{theorem} \label{lemma:disc_error} (MDP Discretization Bound \cite{10.1145/10515.10522}) 
There exists a constant $\mathcal{K}$ (thats depends only on the Lipschitz constant $\mathcal{L}$) such that for some discretisation coarseness $\lambda \in (0, \frac{\mathcal{L}}{2}]$
$$
\norm{V^*_M - V_{M_d}^*}_\infty \leq \frac{\mathcal{K} \lambda}{(1 - \hat{\gamma})^2}.
$$
\end{theorem}

\subsection{Step 6: Adding it together}
\label{step6}
To combine the bounds we obtained in Step 4 and 5, we need an additional bound that bounds the policy value between the continuous and discretised MDP. 

\begin{lemma} [Simulation Lemma for Continuous MDPs]
\label{lem:simulation_original}
Let $M$ and $\hat{M}$ be two MDPs that only differ in $(T, R)$ and $(\hat{T}, \hat{R})$.

Let $\epsilon_{R} \geq \max _{s, a}|\hat{R}(s, a)-R(s, a)|$ and $\varepsilon_{p} \geq \max _{s, a} \| \hat{T}(\cdot \mid  s, a) - T(\cdot \mid s, a)||_{1}$. Then $\forall \pi: \mathbf{\hat{S}} \rightarrow a$,

$$\left\|V_{M}^{\pi}-V_{\hat{M}}^{\pi}\right\|_{\infty} \leq \frac{\varepsilon_{R}}{1-\gamma}+\frac{\gamma \epsilon_{P} V_{\max }}{2(1-\gamma)}.$$

\end{lemma}

Under discretisation, \cite{chow1991optimal} showed that, with a small enough grid size, and restricting to the discretised action space, the difference in transition probability of the continuous MDP $M$ and discretised MDP $M_d$ is upper bounded by a constant.
\begin{lemma}\label{lm:discretised_transition}\cite{chow1991optimal}
    There exists a constant $K_P$ (depending only on constant $\mathcal{L}$) such that 
    $$|T_d(\hat{s}'|\hat{s},\hat{a}_d)-T(\hat{s}'|\hat{s},\hat{a}_d)|\leq K_p \alpha$$
    for all $\hat{s}',\hat{s}\in \hat{S}, \hat{a}_d\in \hat{A}_d$ and all $\alpha\leq (0,\frac{1}{2}\mathcal{L}]$
    \label{lem:transitionbound}
\end{lemma}

We now apply the Lemma \ref{lem:simulation_original} to bound the difference in value for any discretised policy (i.e. restricting action space to $\hat{A}_d$) in the continuous MDP $M$ and discretised MDP $M_d$.

\begin{lemma}
\label{lemma:simulation}
    Let $M_{\hat{A}_d}=(\hat{S},\hat{A}_d,T, R,\gamma)$ be the continuous MDP $M$ restricted to the discretised action space. Recall the discretised MDP $M_d=(\hat{S},\hat{A}_d,T_d, R_d,\gamma)$. Then for any discretised policy $\pi:\hat{S}\rightarrow \hat{A}_d$,
    $$\|V_{M_{\hat{A}_d}}^\pi - V_{M_d}^\pi \|_{\infty}=\|V_{M}^\pi - V_{M_d}^\pi \|_{\infty}\leq \frac{\mathcal{L_R}\alpha}{1-\gamma}+\frac{\gamma K_p \alpha }{(1-\gamma)^2}$$
\end{lemma}
Note that, restricted to discretised policies $\pi$ which only picks actions in $\hat{A}_d$, the value of $\pi$ in the original MDP $M$, $V_M^{\pi}$, equals to its value the same MDP restricted to discretised action space, $V_{M_{\hat{A}_d}}^{\pi}$.


\subsection{Case I}
Summing up the bounds in Lemma \ref{lemma:simulation}, Theorems \ref{thm:pac1} and \ref{lemma:disc_error}, we obtain the final bound for Case I. The final bound guarantees, with high probability, that the policy we obtain from R-FOS is close to the optimal policy in $M$ apart from a constant factor. 

\begin{theorem} 
Suppose that $0 \leq \varepsilon<\frac{1}{1-\gamma}$ and $0 \leq \delta<1$ are two real numbers. Let M be any continuous meta-MDP with inner stochastic game $G = \langle S, A, T_\text{inner}, R_\text{inner} \rangle$.
Let us denote $M_d=\langle \hat{S}_d, \hat{A}_d, T_d, R_d, \gamma\rangle$ as discretised version of $M$ (as described in Case I) using grid size of $\lambda$. There exists inputs $m=m\left(\frac{1}{\varepsilon}, \frac{1}{\delta}\right)$ and $\varepsilon_1$, satisfying 
$$m\left(\frac{1}{\varepsilon}, \frac{1}{\delta}\right)=\tilde{O}\left(\frac{\left(\frac{2 \sqrt{n|S||A|}} {\lambda} + 1 \right)^{n|S||A|}}{\varepsilon^2(1-\gamma)^4}\right)$$
and $\frac{1}{\varepsilon_1}=O\left(\frac{1}{\varepsilon}\right)$, such that if $R\text{-}MAX$ is executed on $M$ with inputs $m$ and $\varepsilon_1$, then the following holds. Let $\mathcal{A}_t$ denote $R\text{-}MAX$'s policy at time $t$ and $\hat{s}_t$ denote the state at time $t$. With probability at least $1-\delta$, 
$$ V_{M}^*\left(\hat{s}_t\right) - V_{M}^{\mathcal{A}_t}\left(\hat{s}_t\right) \leq \varepsilon + \frac{\mathcal{K} \lambda}{(1 - \hat{\gamma})^2} + \frac{\mathcal{L_R}\alpha}{1-\gamma}+\frac{\gamma K_p \alpha }{(1-\gamma)^2}$$
is true for all but
$$
\tilde{O}\left(\frac{\left(\frac{2 \sqrt{n|S||A|}} {\lambda} + 1 \right)^{2n|S||A|} \left(\frac{2 \sqrt{|S||A|}} {\lambda} + 1 \right)^{|S||A|}} {\varepsilon^3(1-\gamma)^6}\right)
$$
timesteps. I.e. the above is the final sample complexity.

\end{theorem}

\subsection{Case II}
Similarly, summing up the bounds in Lemma \ref{lemma:simulation}, Theorems \ref{thm:pac2} and \ref{lemma:disc_error}, we obtain the final bound for Case II. In Section \ref{section:6}, we also show empirically that the number of samples needed indeed scales by a factor of $|S||A|^{2nh}$, as seen in Theorem \ref{thm:case2}.
\begin{theorem} 
\label{thm:case2}
Suppose that $0 \leq \varepsilon<\frac{1}{1-\gamma}$ and $0 \leq \delta<1$ are two real numbers. Let M be any continuous meta-MDP with inner stochastic game $G = \langle S, A, T_\text{inner}, R_\text{inner} \rangle$.
Let us denote $M_d=\langle \hat{S}_d, \hat{A}_d, T_d, R_d, \gamma\rangle$ as discretised version of $M$ (as described in Case I) using grid size of $\lambda$. There exists inputs $m=m\left(\frac{1}{\varepsilon}, \frac{1}{\delta}\right)$ and $\varepsilon_1$, satisfying 
$$m\left(\frac{1}{\varepsilon}, \frac{1}{\delta}\right)=\tilde{O}\left(\frac{(|S||A|)^{nh}}{\varepsilon^2(1-\gamma)^4}\right)$$
and $\frac{1}{\varepsilon_1}=O\left(\frac{1}{\varepsilon}\right)$, such that if $R\text{-}MAX$ is executed on $M$ with inputs $m$ and $\varepsilon_1$, then the following holds. Let $\mathcal{A}_t$ denote $R\text{-}MAX$'s policy at time $t$ and $\hat{s}_t$ denote the state at time $t$. With probability at least $1-\delta$, 
$$ V_{M}^*\left(\hat{s}_t\right) - V_{M}^{\mathcal{A}_t}\left(\hat{s}_t\right) \leq \varepsilon + \frac{\mathcal{K} \lambda}{(1 - \hat{\gamma})^2} + \frac{\mathcal{L_R}\alpha}{1-\gamma}+\frac{\gamma K_p \alpha }{2(1-\gamma)^2}$$
is true for all but
$$
\tilde{O}\left(\frac{(|S||A|)^{2nh} \left(\frac{2 \sqrt{|S||A|}} {\lambda} + 1 \right)^{|S||A|}} {\varepsilon^3(1-\gamma)^6}\right) \text{ timesteps}
$$
\end{theorem}

\section{Experiments} \label{section:6}

We now validate our theoretical findings empirically.

\subsection{The Matching Pennies Environment}
\textbf{Matching Pennies} is a two-player, zero-sum game with a payoff matrix shown in \Cref{tb:mp}. Each agent either pick Heads (H) or Tails (T), $a^i\in\{H, T\}$ and $a^i\sim\pi_{\phi^i}(\cdot\mid\{\})$, where $\phi^i$ correspond to the probability of player $i$ picking H. Note that in this work, the game is not iterated. This means that an inner-episode has a length of 1 and the inner-episodic return corresponds to the payoff after one interaction $r=\text{Payoff Table}(a^1, a^2)$. For R-FOS, this means that a meta-step corresponds to one iteration of the Matching Pennies game. The meta-return corresponds to the discounted, cumulative meta-reward after playing the Matching Pennies $K$ times. While the original M-FOS was evaluated on a more complex, iterated version of the Matching Pennies game, this simple setting with a binary action space is sufficient for our empirical validation. Our setting is also more practical for implementation because the $R_\text{MAX}$ algorithm memory usage grows exponentially with the size of the state and action space.  Thus, for any of the more complex environments from the M-FOS paper we were not able do any empirical analysis of R-FOS at all, due to the exponential sample and memory requirements.

\begin{table}[!htb]
\centering
\begin{tabular}{c|cc}
Player 1\textbackslash Player 2 & Head     & Tail     \\ \hline
Head                            & (+1, -1) & (-1, +1) \\ 
Tail                            & (-1, +1) & (+1, -1) \\ 
\end{tabular}
\label{tb:mp}
\caption{Payoff Matrix for MP}
\end{table}

\subsection{\textbf{Experiment Setup}}

We implement an empirical version of our R-FOS algorithm. Because the R-FOS algorithm uses Q-value iteration to solve the meta-game, the algorithm needs to keep a copy of the meta-Q-value table. Therefore, memory usage grows exponentially with respect to the inner-game's state-action space size. We found that Case I of the algorithm was intractable to implement even with a compact environment like MP. The meta Q-table of size $|\hat{S}| \times |\hat{A}|$ was simply too large to fit in memory. Therefore, we focus on empirically validating a simplified case of Case II. We make two simplifications,
\begin{enumerate}
    \item The meta-state uses a partial history of past actions. Only the most recent $h$ actions are used, where $h$ is a hyper-parameter we pick. The window size allows us to control the size of the meta-game state, i.e., $\hat{{S}}\in\mathbb{R}^{2^h}$. Because the MP game only has one state, it is not necessary to include the state. 
    \item To further decrease the problem size for tractability, we define the meta-agent action to be the inner-agent's greedy action, instead of the Q-table. This results in a much small meta-action size of $|\hat{A}| = 2$
\end{enumerate}


\section{Results and Discussion}\label{section:7}

\begin{figure}[!htb]
    \centering\includegraphics[width=0.99\columnwidth]{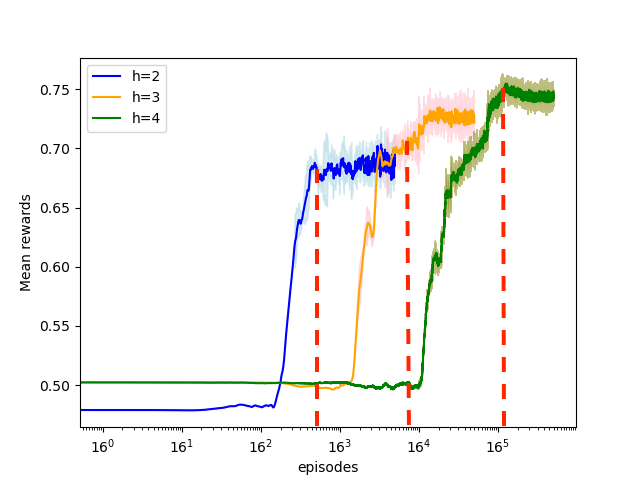}
    \caption{Empirical sample complexity while varying the trajectory window $h$. We plot the meta-reward per meta-episode. To better visualise the connection with the theory results, we plot the x-axis in $\log_{16}$ scale. The reported results are the mean over 3 seeds with standard error.}
    \label{fig:rewc}
\end{figure}

We draw connections between our sample complexity theory results and experimental results in the MP environment. Our goal is to analyse the scaling law of R-FOS. Specifically, we investigate how the sample complexity changes when we vary the window-size $h$. Under the MP environment settings, the inner-game state-action space size is $|{S}||{A}| = 2$  and the number of players is $n=2$.
Following the bound in Theorem \ref{thm:case2}, we see that the only term that depends on h is the $16^h$ term:
\vspace{-0.005pt}
$$
\tilde{O}\left(\frac{(|S||A|)^{2nh} \left(\frac{2 \sqrt{|S||A|}} {\lambda} + 1 \right)^{|S||A|}} {\varepsilon^3(1-\gamma)^6}\right) \sim
\tilde{O}\left(\frac{16^{h} \left(\frac{2 \sqrt{2}} {\lambda} + 1 \right)^2} {\varepsilon^3(1-\gamma)^6}\right)
$$

Hence, our theory results says that whenever the game horizon is increased by 1, we expect to see the sample complexity to increase by a factor of 16 in the MP environment.  
Figure \ref{fig:rewc} shows the reward across the meta episodes on a log $16$ scale. The graph contains three reward curves for meta-trajectory length $h = 2,3,4$, which converges approximately at $16^3, 16^4, and 16^5$ episodes. Indeed, this is consistent with our theoretical results in Theorem \ref{thm:pac2}.

\section{Conclusion}
We presented three main contributions in our work. First of all, we presented R-FOS, a tabular algorithm adapted from M-FOS. Unlike M-FOS, which learns a policy in a continuous meta-MDP, R-FOS instead learns a policy in a discrete approximation of the original meta-MDP which allows us to more easily perform theoretical analysis.  Within this discretised meta-MDP, R-FOS uses the $R_\text{MAX}$ algorithm as the meta-agent. We adapted R-FOS from M-FOS such that it still maintains all key attributes of M-FOS. Second of all, we derived an exponential sample complexity bound for both cases described in M-FOS (the two cases being either inner-game policies or inner-game trajectory history as meta-state). Specifically, we proved that with high probability, the policy learnt by R-FOS is close to the optimal policy from the original meta-MDP up to a constant distance. Finally, we implemented R-FOS and investigated the empirical sample complexity in the Matching Pennies environment. We draw connections between theory and experiments by showing both results scales exponentially according to the size of the inner-game's state-action-space.



\begin{acks}
\texttt{KF was supported by the D. H. Chen Foundation Scholarship. QZ is supported by Armasuisse and Cohere.} 

\end{acks}



\clearpage
\newpage
\bibliographystyle{ACM-Reference-Format} 
\bibliography{Citation}

\clearpage
\newpage
\onecolumn
\appendix
\appendix

\section{Overview}

\subsection{Problem Setup}

The \textbf{meta-game} is defined by a continuous MDP $M = \langle\hat{S},\hat{A}, T, R, \gamma\rangle$ with finite horizon $h$.

For the remaining of the proof, we consider two ways to formulate the meta-state space and meta-action space:
\begin{itemize}
    \item In Case I, the meta-state is all inner agents' policies' parameters from the previous timestep, and the meta-action is the inner agent's current policy parameters.
    \item In Case II, the only difference with Case I is the meta-state is instead all inner agents' trajectories. 
\end{itemize}
See Table \ref{tb:space_representation} for a summary for the two cases.

\begin{table}[!htb]
\centering
\renewcommand{\arraystretch}{1.2} 
\setlength{\tabcolsep}{20pt} 
\caption{Two cases of representing the meta-state space and meta-action space.}
\begin{tabular}{|c|c|c|}
\hline
 & $\hat{s}_t =$ & $\hat{a}_t =$ \\ \hline
Case I & $\bm{\phi}_{t-1} = [\phi_{t-1}^{i}, \bm{\phi_{t-1}^{-i}}]$ & $\phi^{i}_t$ \\ \hline
Case II & $\bm{\tau}_t$ & $\phi^{i}_t$ \\ \hline
\end{tabular}
\label{tb:space_representation}
\end{table}

The \textbf{inner-game} is an n-player fully-observable discrete stochastic game $G = \langle S, A, T_\text{inner}, R_\text{inner} \rangle$ with finite horizon $h$.

\subsection{Theory Overview}
We derive the sample complexity of our R-FOS algorithm. On a high-level, the proof consists of six steps.
\begin{enumerate}
    \item We discretise our MDP $M$ into $M_d$. We first descretise the continuous meta-state space and meta-action space using epsilon-nets \cite{chow1991optimal}. Based on the descretised meta-state and meta-action space, we then define the discretised transition and reward function. 
    \item We then construct a $m$-known discretised MDP $M_{m,d}$, as described by the R-MAX algorithm\cite{pacmdp}.
    \item Then, we estimate the empirical $m$-known discretised MDP $\hat{M}_{m,d}$ using maximum likelihood estimate. This is the same procedure described by the R-MAX algorithm\cite{pacmdp}. Our algorithm, R-FOS, learns an optimal policy in $\hat{M}_{m,d}$.
    \item We first prove a PAC-bound between the optimal policies learnt in $\hat{M}_{m,d}$ and $M_d$. This step uses results from \citep{pacmdp}.
    \item We then prove a bound between the optimal policies learnt in $M_d$ and $M$. This step uses results from \citep{chow1991optimal}.
    \item We obtain the final PAC-bound building from the two bounds from above. 
\end{enumerate}

\section{Nomenclature}

\begin{table}[!htb]
\label{tb:notation}
\begin{tabular}{p{0.13\textwidth} p{0.85\textwidth}}
\toprule
Symbol & Definition \\
\midrule
$\hat{s}_t = \hat{s}$             & Meta-state at time $t$, time subscript $t$ is omitted for convenience  \\
$\hat{a}_t = \hat{a}$             & Meta-action at time $t$, time subscript $t$ is omitted for convenience \\
$(\hat{s}_d, \hat{a}_d)$          & discretised state-action pair in meta-game \\
$\hat{r}_d = \hat{r}(\hat{s}_d, \hat{a}_d)$   & Meta-reward function parameterised by discretised meta-state-action pair \\
$\phi^i_t = \phi^i$               & The set of inner-game policy parameters of our agent at time $t$, time subscript $t$ is omitted for convenience  \\
$\phi^i_{t,d} = \phi^i_d$           & The set of discretised inner-game policy parameters of our agent at time $t$, time subscript $t$ is omitted for convenience\\
$\bm{\phi}^{-i}$                  & The set of inner-game policy parameters of all agents except our agent at time $t$, time subscript $t$ is omitted for convenience \\
$\bm{\phi^{-i}_d}$      & The set of discretised inner-game policy parameters of all agents except our agent at time $t$, time subscript $t$ is omitted for convenience \\
$\hat{R}(\hat{s}, \hat{a}), \hat{T}(\hat{s}, \hat{a})$ & Empirical estimate of reward and transition distribution \\
$R(\hat{s}, \hat{a}), T(\hat{s}, \hat{a})$ & True reward and transition distribution \\

\bottomrule
\end{tabular}
\vskip -0.1in
\end{table}

\section{\textbf{Assumptions}}
We first outline all assumptions made in deriving the sample complexity of the R-FOS algorithm.

To establish the bound in step 5, we make the following assumptions.


\begin{assumption} \label{ass:timestep}
Both meta-game and inner-game are finite horizon. We use $h$ to denote the meta-game horizon, and $h_{\text{inner}}$ to denote the inner-game horizon.
\end{assumption}

\begin{assumption} \label{ass:gamma}
The meta-game uses a discount factor of $\gamma$. For simplicity of the proof, assume the inner-game uses a discount factor of $1$. Although this assumption can be easily omitted by substituting $R_\text{MAX}$ in the original proof in \citep{pacmdp}.
\end{assumption}

\begin{assumption}
\label{ass:bounded_reward}
We assume the inner-game reward is bounded. For simplicity of the proof, we set this bound as $\frac{1}{h_\text{inner}}$, where $h$ is the horizon of the inner game. Formally, for all $(s, a), 0 \leq R_\text{inner}(s, a) \leq \frac{1}{h}$. This allows us to introduce the notion of maximum reward and maximum value function as $R_{\max, \text{inner}} = \frac{1}{h_\text{inner}}$ and $V_{\max, \text{inner}} = 1$ respectively. This implies the reward and value function in the meta-game are also bounded, i.e., $R_{\max} = 1$ and $V_{\max} = \frac{1}{1-\gamma}$.

\end{assumption}

\begin{assumption} \label{ass:pomdp}
The inner game is assumed to be discrete.
\end{assumption}

To establish the bound in step 6, we make the following assumptions.

\begin{assumption} \label{ass:lips_rew}
The meta-reward function is Lipschitz-continuous: For all $ \hat{s}_1, \hat{s}_2\in \hat{S} \text{ and } \hat{a}_1, \hat{a}_2 \in \hat{A},$
$$
\abs{R(\hat{s}_1, \hat{a}_1) - R(\hat{s}_2, \hat{a}_2)} \leq \mathcal{L}_R \norm{(\hat{s}_1, \hat{a}_1) - (\hat{s}_2, \hat{a}_2)}_\infty, \quad \forall \hat{s}_1, \hat{s}_2 \in \hat{S} \text{ and } \hat{a}_1, \hat{a}_2 \in \hat{A}
$$
where $\mathcal{L}_R$ is the meta-reward function's Lipschitz-constant.
\end{assumption}

\begin{assumption} \label{ass:lips_trans}
The meta-transition function is Lipschitz-continuous: For all $ \hat{s}_1, \hat{s}_1^\prime, \hat{s}_2\in \hat{S} \text{ and } \hat{a}_1, \hat{a}_1^\prime, \hat{a}_2 \in \hat{A},$
$$
\abs{T(\hat{s}_1^{\prime} \mid \hat{s}_1, \hat{a}_1) - T(\hat{s}_2^{\prime} \mid \hat{s}_2, \hat{a}_2)} \leq \mathcal{L}_T \norm{(\hat{s}_1^\prime, \hat{s}_1, \hat{a}_1) - (\hat{s}_2^\prime, \hat{s}_2, \hat{a}_2)}_\infty
$$
where $\mathcal{L}_T$ is the meta-transition function's Lipschitz-constant.
\end{assumption}

\begin{assumption} \label{ass:point_to_set}
There's a Lipschitz-continuous point-to-set mapping between meta-state space and meta-action space such that for any $\hat{s}, \hat{s}^\prime \in \hat{S}$ and $\hat{a}, \hat{a}^\prime \in \hat{A}$, there exists some $\hat{a} \in U(\hat{s})$ such that $\norm{\hat{a} - \hat{a}^\prime}_\infty < \mathcal{L} \norm{\hat{s} - \hat{s}^\prime}_\infty$
\end{assumption}

\begin{assumption} \label{ass:pdf}
The meta-game transition function $T(\cdot \mid \hat{s}, \hat{a})$ is a probability density function such that $0 \leq T(\hat{s}^\prime \mid \hat{s}, \hat{a}) \leq \mathcal{L} \text{ and } \int_{\hat{S}} T(\hat{s}^\prime \mid \hat{s}, \hat{a}) \, d\hat{s}^\prime = 1, \quad \forall \hat{s}, \hat{s}^\prime \in \hat{S} \text{ and } \hat{a} \in \hat{A}$
\end{assumption}

\section{\textbf{Step 1: Discretisation with $\varepsilon$-Net}}
\begin{figure}[h]
\centering
\includegraphics[width=0.27\columnwidth]{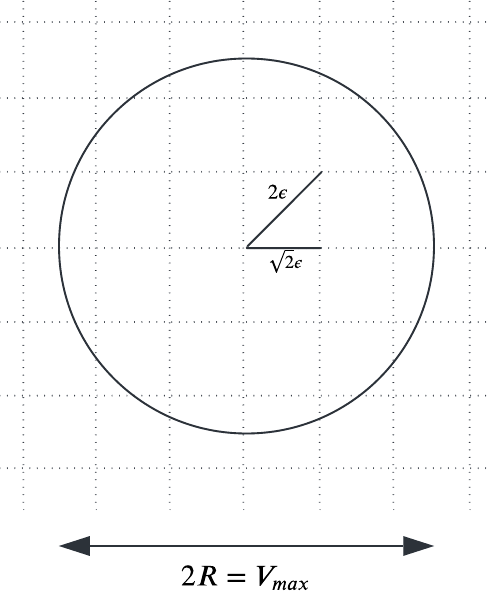}
\caption{$\varepsilon$-Net for $\Theta = \{\theta \in \mathbb{R}^2: \norm{\theta} \leq R\}$ \cite{net}}
\label{fig:net}
\end{figure}

\label{app:epsilonnet}
To apply the R-MAX algorithm, we first convert the MDP $M$ in M-FOS into a tabular MDP $M_d$.

\begin{definition}($\varepsilon$-Net \cite{net}) For $\varepsilon>0$, $\mathcal{N}_\varepsilon$ is an $\varepsilon$-net over the set $\Theta \subseteq \mathbb{R}^D$ if for all $\theta \in \Theta$, there exists $\theta' \in \mathcal{N}_\varepsilon$ such that $\norm{\theta -\theta'}_2 \leq \varepsilon$.

\end{definition}

To discretise a $D$-dimensional sphere of radius $R$, we use a $\varepsilon$-net containing $D$-dimensional cubes of sides $\lambda$. 
This results in a $\left(\frac{2R}{\lambda} + 1\right)^D$ points. Within each $D$-dimensional cube, the largest distance between the vertices and the interior points comes from the center of the cube, which is $\frac{\lambda\sqrt{d}}{2}$. Therefore, to guarantee a full cover of all the points in the sphere, the largest cube size that we can have should satisfy $\varepsilon = \frac{\lambda\sqrt{d}}{2}$. \Cref{fig:net} illustrates an example of using $\varepsilon$-nets to discretise the input space $\Theta = \{\theta \in \mathbb{R}^2: \|\theta\| \leq R\}$. From here on, we will replace the $\varepsilon$ in $\varepsilon$-net with $\alpha$ to avoid notation overloading.

\subsection{Discretisation of the State and Action Space: Case I}
In \textit{Case I}, the meta-state $\hat{s}_{t}$ is all inner agents' policies parameters from the previous timestep. Each of the inner agent $i$'s policy is a Q-table, denoted as $\phi_i \in \mathbb{R}^{|S| \times |A|}$.
Formally, $\hat{s}_{t} := \bm{\phi}_{t-1} = [\phi_{t-1}^{i}, \bm{\phi_{t-1}^{-i}}].$ The meta-action $\hat{a}_t$ is the inner agent's current policy parameters $\phi^{i}_t$. 

For the meta-action space $\hat{A}$ and a chosen discretisation error $\alpha>0$ , we obtain the $\varepsilon$-net $\hat{A}_d \subset \hat{A}$ such that for all $\hat{a} \in \hat{A}$, there exist $\hat{a}_d \in \hat{A}_d$ where 
\begin{equation}
\label{eq:alpha}
    \norm{\hat{a} - \hat{a}_d} \leq \alpha.
\end{equation}

We can infer $\hat{A}$ has dimension $D = |S| |A|$ and Radius $R =  \sqrt{ 1 \cdot |S||A|} = \sqrt{|S||A|}$ (Assumption \ref{ass:bounded_reward}). Dividing the space with grid size $\lambda$ results in the size of discretised meta-action space upper to be bounded by
\begin{equation}
    \label{eq:metaactionsize}
    |\hat{A}_d| \leq \left(\frac{2 \sqrt{|S||A|}} {\lambda} + 1 \right)^{|S||A|}.
\end{equation}

Similarity, we can also infer $\hat{S}$ has dimension $D = n|S||A| $ and Radius $R =  \sqrt{ 1 \cdot n|S||A|} = \sqrt{n|S||A|}$.
Dividing the space with grid size $\lambda$ results in the size of discretised meta-state space upper to be bounded by
\begin{equation}
    \label{eq:metastatesize1}
    |\hat{S}_d| \leq \left(\frac{2 \sqrt{n|S||A|}} {\lambda} + 1 \right)^{n|S||A|}.
\end{equation}

\subsubsection{Discretisation of the State and Action Space: Case II}

In Case II, the meta-state $\hat{s}_t$ is all inner agents' trajectories. Formally, $\hat{s}_{t} := \bm{\tau}_t.$, where $\tau^i_t = \{s_0, a_0, s_1, a_1, ..., s_t, a_t\}$. Because we assume the Inner-Game is discrete (i.e. the state and action space are both discrete), the meta-state in this case does not need discretisation. Thus, we can directly obtain the meta-state space, which is,

\begin{equation}
\label{eq:metastatesize2}
    |\hat{S}_t| = (|S||A|)^{nh}.
\end{equation}

The meta-action remains same as Case I.

\subsection{Discretisation of the Transition and Reward Function}
Under the above discretisation procedure, we define the discretised MDP $M_d=(\hat{S}, \hat{A}_d, T_d, R_d, \gamma)$, where the state space remains continuous, the action space is restricted to discretised actions. We define the transition function and reward function for $M_d$ as:

\begin{equation}
    \label{eq:discrete_transition}
    T_d(\hat{s}'|\hat{s},\hat{a}_d) = \frac{T(\hat{s}'|\hat{s}_d, \hat{a}_d)}{\int_{\hat{S}} T(\hat{z}_d|\hat{s}_d, \hat{a}_d)d\hat{z}}
\end{equation}
We can view $T_d(\hat{s}'|\hat{s}_d,\hat{a}_d)$ as a normalized sample of $T_d(\cdot|\cdot,\hat{a}_d)$ at $\hat{s}',\hat{s}_d$.

\begin{equation}
    \label{eq:discrete_reward}
    R_d(\hat{s},\hat{a}_d) = R(\hat{s}_d, \hat{a}_d)
\end{equation}


\section{Step 2: The $m$-known Discretised MDP}
In the last step, we converted the meta-MDP $M$ into a discretised meta-MDP $M_d$. From $M_d$, R-FOS builds a $m$-known discretised MDP $M_{m,d}$. 
\begin{definition}[m-Known MDP]
\label{def:known}
Let $M_d=\langle\hat{S}_d, \hat{A}_d, T_d, R_d, \gamma\rangle$ be a MDP. We define $M_m,d$ to be the $m$-known MDP (See \Cref{tb:theMs_appendix}), where $m$-known is the set of state-action pairs that has been visited at least $m$ times. For all state-action pairs in $m$-known, the induced MDP $M_{m,d}$ behaves identical to $M_d$. For state-action pairs outside of $m$-known, the state-action pairs are self-absorbing and maximally rewarding. 

\end{definition}

\begin{table}[!htb]
\centering
\renewcommand{\arraystretch}{1.5}
\begin{tabular}{|c|c|c|c|c|}
\hline
 & \multicolumn{1}{p{2cm}|}{\centering Ground Truth MDP $M$} & \multicolumn{1}{p{2cm}|}{\centering Discretised MDP $M_d$} & \multicolumn{1}{p{2.5cm}|}{\centering $m$-known Discretised MDP $\hat{M}_{m,d}$} & \multicolumn{1}{p{2.5cm}|}{\centering Empirical $m$-known Discretised MDP $\hat{M}_{m,d}$} \\
\hline
Known & = $M$ & = $M_d$ & = $M_d$ & $\approx M_d$ \\
\hline
Unknown & = $M$ & = $M_d$ & \multicolumn{2}{c|}{self-loop with maximum reward} \\
\hline
\end{tabular}
\caption{Relationship between $M, M_d, M_{m,d}, \hat{M}_{m,d}$}
\label{tb:theMs_appendix}
\end{table}

\section{Step 3: The Empirical Discretised MDP}

\begin{definition}[Empirical m-Known MDP]
\label{def:known}
$M_m$ is the expected version of $\hat{M}_m$ where:

\begin{center}
\begin{equation}
\begin{aligned}
    T_{m,d}(\hat{s}_d^{\prime} \mid \hat{s}_d, \hat{a}_d) & :=\left\{\begin{array}{lll} T_d(\hat{s}_d^{\prime} \mid \hat{s}_d, \hat{a}_d) & \text { if } (\hat{s}_d, \hat{a}_d) \in \text{m-known}\hfill \\ \mathbbm{1}[\hat{s}_d^{\prime}=\hat{s}_d], & \text { otherwise}\hfill & \end{array}\right.\\
    \hat{T}_{m,d}(\hat{s}_d^{\prime} \mid \hat{s}_d, \hat{a}_d) & := \begin{cases} \frac{n(\hat{s}_d, \hat{a}_d, \hat{s}_d^{\prime})}{n(\hat{s}_d, \hat{a}_d)}, & \text { if }(\hat{s}_d, \hat{a}_d) \in \text{m-known}\hfill \\ \mathbbm{1}[\hat{s}_d^{\prime}=\hat{s}_d], & \text { otherwise }\hfill\end{cases} \\
     R_{m,d}(\hat{s}_d, \hat{a}_d)& :=\left\{\begin{array}{lll} R_d(\hat{s}_d, \hat{a}_d), & \text { if } (\hat{s}_d, \hat{a}_d) \in \text{m-known}\hfill \\ R_{\max} & \text { otherwise}\hfill\end{array}\right. \\
     \hat{R}_{m,d}(\hat{s}_d, \hat{a}_d)&= \begin{cases}\frac{\sum_{i}^{n(\hat{s}_d, \hat{a}_d)}r(\hat{s}_d, \hat{a}_d)}{n(\hat{s}_d, \hat{a}_d)}, & \text { if }(\hat{s}_d, \hat{a}_d) \in \text{m-known}\hfill \\ R_{\max}, & \text { otherwise }\hfill\end{cases} \\
\end{aligned}
\end{equation}
\end{center}

\end{definition}
$\hat{R}_{m,d}(\hat{s}_d, {\hat{a}_d})$ and $\hat{T}_{m,d}\left(\hat{s}_d^{\prime} \mid \hat{s}_d, {\hat{a}_d}\right)$ are the maximum-likelihood estimates for the reward and transition distribution of state-action pair $(s_d,a_d)$ with $n(s_d,a_d) \geq m$ observations of $(s_d,a_d)$.

\section{Step 4: The Bound Between $M_d$ and $\hat{M}_{m,d}$}

\begin{theorem} (R-MAX MDP Bound \cite{pacmdp})
\label{eq:strehl}
Suppose that $0 \leq \varepsilon<\frac{1}{1-\gamma}$ and $0 \leq \delta<1$ are two real numbers and $M=\langle S, A, T, R, \gamma\rangle$ is any MDP. There exists inputs $m=m\left(\frac{1}{\varepsilon}, \frac{1}{\delta}\right)$ and $\varepsilon_1$, satisfying $m\left(\frac{1}{\varepsilon}, \frac{1}{\delta}\right)=O\left(\frac{(|S|+\ln (|S| |A| / \delta)) V_{\max }^2}{\varepsilon^2(1-\gamma)^2}\right)$ and $\frac{1}{\varepsilon_1}=O\left(\frac{1}{\varepsilon}\right)$, such that if $R\text{-}MAX$ is executed on $M$ with inputs $m$ and $\varepsilon_1$, then the following holds. Let $A_t$ denote $R\text{-}MAX$'s policy at time $t$ and $s_t$ denote the state at time $t$. With probability at least $1-\delta$, $V_M^{A_t}\left(s_t\right) \geq V_M^*\left(s_t\right)-\varepsilon$ is true for all but
$$
\tilde{O}\left(|S|^2 |A| /\left(\varepsilon^3(1-\gamma)^6\right)\right)
$$
timesteps.

\end{theorem}

\subsection{Case I}
\begin{theorem} 
\label{thm:pac1}
Suppose that $0 \leq \varepsilon<\frac{1}{1-\gamma}$ and $0 \leq \delta<1$ are two real numbers. Let M be any continuous meta-MDP with inner stochastic game $G = \langle S, A, T_\text{inner}, R_\text{inner} \rangle$.
Let us denote $M_d=\langle \hat{S}_d, \hat{A}_d, T_d, R_d, \gamma\rangle$ as discretised version of $M$ (as described in Case I) using grid size of $\lambda$. There exists inputs $m=m\left(\frac{1}{\varepsilon}, \frac{1}{\delta}\right)$ and $\varepsilon_1$, satisfying 

$$m\left(\frac{1}{\varepsilon}, \frac{1}{\delta}\right)=\tilde{O}\left(\frac{\left(\frac{2 \sqrt{n|S||A|}} {\lambda} + 1 \right)^{n|S||A|}}{\varepsilon^2(1-\gamma)^4}\right)$$

and $\frac{1}{\varepsilon_1}=O\left(\frac{1}{\varepsilon}\right)$, such that if $R\text{-}MAX$ is executed on $M$ with inputs $m$ and $\varepsilon_1$, then the following holds. Let $\mathcal{A}_t$ denote $R\text{-}MAX$'s policy at time $t$ and $\hat{s}_t$ denote the state at time $t$. With probability at least $1-\delta$, $ V_{M_d}^*\left(\hat{s}_t\right) - V_{M_d}^{\mathcal{A}_t}\left(\hat{s}_t\right) \leq \varepsilon$ is true for all but
$$
\tilde{O}\left(\frac{\left(\frac{2 \sqrt{n|S||A|}} {\lambda} + 1 \right)^{2n|S||A|} \left(\frac{2 \sqrt{|S||A|}} {\lambda} + 1 \right)^{|S||A|}} {\varepsilon^3(1-\gamma)^6}\right)
$$
timesteps.

\end{theorem}

\begin{proof}
We first plug in Equations \ref{eq:metastatesize1} and \ref{eq:metaactionsize} into Theorem \ref{eq:strehl}. Dropping the logarithm terms and plugging in $V_\text{max} = \frac{1}{1-\gamma}$, we obtain the results.
\end{proof}

\subsection{Case II}

\begin{theorem} 
\label{thm:pac2}
Suppose that $0 \leq \varepsilon<\frac{1}{1-\gamma}$ and $0 \leq \delta<1$ are two real numbers. Let M be any continuous meta-MDP with inner stochastic game $G = \langle S, A, T_\text{inner}, R_\text{inner} \rangle$.
Let us denote $M_d=\langle \hat{S}_d, \hat{A}_d, T_d, R_d, \gamma\rangle$ as discretised version of $M$ (as described in Case I) using grid size of $\lambda$. There exists inputs $m=m\left(\frac{1}{\varepsilon}, \frac{1}{\delta}\right)$ and $\varepsilon_1$, satisfying 

$$m\left(\frac{1}{\varepsilon}, \frac{1}{\delta}\right)=\tilde{O}\left(\frac{(|S||A|)^{nh}}{\varepsilon^2(1-\gamma)^4}\right)$$

and $\frac{1}{\varepsilon_1}=O\left(\frac{1}{\varepsilon}\right)$, such that if $R\text{-}MAX$ is executed on $M$ with inputs $m$ and $\varepsilon_1$, then the following holds. Let $\mathcal{A}_t$ denote $R\text{-}MAX$'s policy at time $t$ and $\hat{s}_t$ denote the state at time $t$. With probability at least $1-\delta$, $ V_{M_d}^*\left(\hat{s}_t\right) - V_{M_d}^{\mathcal{A}_t}\left(\hat{s}_t\right) \leq \varepsilon$ is true for all but
$$
\tilde{O}\left(\frac{(|S||A|)^{2nh} \left(\frac{2 \sqrt{|S||A|}} {\lambda} + 1 \right)^{|S||A|}} {\varepsilon^3(1-\gamma)^6}\right)
$$
timesteps.

\end{theorem}

\begin{proof}
We first plug in Equations \ref{eq:metastatesize2} and \ref{eq:metaactionsize} into Theorem \ref{eq:strehl}. Dropping the logarithm terms and plugging in $V_\text{max} = \frac{1}{1-\gamma}$, we obtain the results.
\end{proof}

\section{Step 5: The Bound between $M$ and $M_d$}
\begin{theorem} \label{lemma:disc_error} (MDP Discretization Bound \cite{10.1145/10515.10522}) 
There exists a constant $\mathcal{K}$ (thats depends only on the Lipschitz constant $\mathcal{L}$) such that for some discretisation coarseness $\lambda \in (0, \frac{\mathcal{L}}{2}]$ such that
$$
\norm{V^*_M - V_{M_d}^*}_\infty \leq \frac{\mathcal{K} \lambda}{(1 - \hat{\gamma})^2}.
$$
\end{theorem}

\subsection{\textbf{Sample Complexity Analysis}} \label{section:samplec}

\section{Step 6: Adding it together}
\begin{lemma} [Simulation Lemma for Continuous MDP]
\label{lem:simulation_original}
Let $M$ and $\hat{M}$ be two MDPs that only differ in $(T, R)$ and $(\hat{T}, \hat{R})$. And suppose the common state space $\hat{S}$ of $M$ and $\hat{M}$ is continuous, and denote the common action space as $\hat{A}$. 

Let $\epsilon_{R} \geq \max _{s, a}|\hat{R}(s, a)-R(s, a)|$ and $\varepsilon_{p} \geq \max _{s, a} \| \hat{T}(\cdot \mid  s, a) - T(\cdot \mid s, a)||_{1}$\footnote{Note that given $T(\cdot|s,a),\hat{T}(\cdot|s,a)$ are functions on continuous space $\hat{S}$,  this is the $L_1$ norm defined by $\|f\|_1=\int_{\hat{S}}|f|d\mu$. Similarly, throughout the proof, we denote as $\|\|_p$ the $L_p$ norm, defined by  $\|f\|_p=(\int_{\hat{S}}|f|^p d\mu)^{1/p}$; and the inner product $\langle f,g\rangle= \int_{\hat{S}} fg d\mu$.}. Then $\forall \pi: \hat{S} \rightarrow \hat{A}$,

$$\left\|V_{M}^{\pi}-V_{\hat{M}}^{\pi}\right\|_{\infty} \leq \frac{\varepsilon_{R}}{1-\gamma}+\frac{\gamma \epsilon_{P} V_{\max }}{2(1-\gamma)}.$$

\end{lemma}

\begin{proof}
For all $s \in \hat{S}$, 

\begin{equation}
    \label{eq:sim1}
    \begin{aligned}
    \left| V_{\hat{M}}^{\pi}(s)-V_{M}^{\pi}(s)\right| & =\left| \hat{R}(s, \pi)+ \gamma \left\langle\hat{T}(\cdot, \pi), V_{\hat{M}}^{\pi}\right\rangle - R(s, \pi)-\gamma \left\langle T(\cdot, \pi), V_{M}^{\pi} \right\rangle \right|  \\
    &\leq |\hat{R}(s, \pi)-R(s, \pi)|+\gamma\left|\left\langle\hat{T}(\cdot, \pi), V_{\hat{M}}^{\pi}\right\rangle-\left\langle T(\cdot, \pi), V_{M}^{\pi}\right\rangle\right| \quad\quad \text{(triangular inequality)}\\
    &\leq \varepsilon_{R} + \gamma \left( \left| \left\langle\hat{T}(\cdot, \pi), V_{\hat{M}}^{\pi}\right\rangle  {-\left\langle T(\cdot, \pi), V_{\hat{M}}^{\pi}\right\rangle+\left\langle T(\cdot, \pi), V_{\hat{M}}^\pi \right\rangle} -\left\langle T(\cdot, \pi), V_{M}^{\pi}\right\rangle \right| \right) \quad \text{(add \& subtract)}\\
    &\leq \varepsilon_{R}+\gamma\left|\left\langle\hat{T}(\cdot, \pi)-T(\cdot, \pi), V_{\hat{M}}^{\pi}\right\rangle\right|+\gamma\left| \left\langle T(\cdot, \pi), V_{\hat{M}}^{\pi}-V_{M}^{\pi}\right\rangle \right| \\
    & \leq \varepsilon_{R}+ \gamma \left| \left\langle\hat{T}(\cdot, \pi)-T(\cdot, \pi), V_{\hat{M}}^{\pi}\right\rangle \right|+\gamma{ \left|\left| V_{\hat{M}}^\pi - V_M ^\pi \right|\right|_{\infty}}.
    \end{aligned}
\end{equation}

Since Equation \ref{eq:sim1} holds for all $\hat{s} \in \hat{S}$, we can take the infinite-norm on the left hand side:

\begin{equation}
    \label{eq:simulation_lemma}
    \begin{aligned}
    \left| \left| V_{\hat{M}}^{\pi}-V_{M}^{\pi}\right| \right|_{\infty} \leq \varepsilon_{R}+ \gamma {\left|\left\langle\hat{T}(\cdot, \pi)-T(\cdot, \pi), V_{\hat{M}}^{\pi}\right\rangle\right|}+\gamma{ \left|\left| V_{\hat{M}}^\pi - V_M ^\pi \right|\right|_{\infty}}.
    \end{aligned}
\end{equation}

We then expand the middle term as follows:
\begin{equation}
    \label{eq:inner_product}
    \begin{aligned}
  { \left|\left\langle\hat{T}(\cdot, \pi) - T(\cdot, \pi), V_{\hat{M}}^{\pi}\right\rangle\right|} &=\left| \left\langle\hat{T}(\cdot, \pi)-T(\cdot, \pi), V_{\hat{M}}^{\pi}-\mathbf{1} \cdot \frac{R_\text{max}}{2(\mathbf{1}-\gamma)}\right\rangle\right| \quad\quad (\text{where $\mathbf{1}$ is a vector of ones} \in \mathbbm{R}^{|S|})\\
  & \leq\|\hat{T}(\cdot, \pi)-T(\cdot, \pi)\|_{1} \cdot\left\|V_{\hat{M}}^{\pi}-\mathbf{1}
  \cdot \frac{R_{\max }}{2(1-\gamma)}\right\|_{\infty} \quad\quad (\text{Holder's inequality}) \\
  & \leq \epsilon_P \cdot \frac{R_\text{max}}{2(1-\gamma)} \\
  & = \epsilon_P \cdot \frac {V_\text{max}}{2}.
    \end{aligned}
\end{equation}

In Equation \ref{eq:inner_product}, the first step shifts the range of $V$ from $[0, \frac{R_\text{max}}{(1-\gamma)}]$ to $[-\frac{R_\text{max}}{2(1-\gamma)}, \frac{R_\text{max}}{2(1-\gamma)}]$ to obtain a tighter bound by a factor of 2. The equality in line 1 holds because of the following, where $C$ is any constant:

\begin{equation}
\label{eq:intermediate}
\begin{aligned}
     \langle \hat{T} - T, C \cdot \mathbf{1} \rangle &= C \langle \hat{T} - T, \mathbf{1}
    \rangle \\
    &= C \left(\langle \hat{P}, \mathbf{1} \rangle - \langle P, \mathbf{1} \rangle \right) \\
    &= C (1-1) \quad \quad \text{because $P$ and $\hat{P}$ are probability distributions} \\
    &= 0
    \end{aligned}    
\end{equation}
From equation \ref{eq:intermediate}, we observe the equality in line 1 holds:
\begin{equation}
    \begin{aligned}
    \langle \hat{T} - T, V - C \cdot \mathbf{1} \rangle &= \langle \hat{T} - T, V \rangle - \langle \hat{T} - T, C \cdot \mathbf{1} \rangle \\
    &=  \langle \hat{T} - T, V \rangle - 0 \\
    & =  \langle \hat{T} - T, V\rangle.
    \end{aligned}
\end{equation}

Finally, we plug equation \ref{eq:inner_product} into equation \ref{eq:simulation_lemma} to obtain the bound.

\begin{equation}
    \begin{aligned}
    \left|\left| V_{\hat{M}}^{\pi}-V_{M}^{\pi}\right|\right|_{\infty} & \leq \epsilon_R + \gamma \epsilon_P \cdot \frac {V_\text{max}}{2} + \gamma \left|\left|  V_{\hat{M}}^{\pi}-V_{M}^{\pi}\right|\right|_{\infty} \\
    &= \frac{\epsilon_R}{1-\gamma} +  \frac {\gamma \epsilon_P V_\text{max}}{2(1-\gamma)}.
    \end{aligned}
\end{equation}

\end{proof}

Under discretisation, \cite{chow1991optimal} showed that, with a small enough grid size, and restricting to the discretised action space, the difference in transition probability of the continuous MDP $M$ and discretised MDP $M_d$ is upper bounded by a constant.
\begin{lemma}\label{lm:discretised_transition}\cite{chow1991optimal}
    There exists a constant $K_P$ (depending only on constant $\mathcal{L}$) such that 
    $$|T_d(\hat{s}'|\hat{s},\hat{a}_d)-T(\hat{s}'|\hat{s},\hat{a}_d)|\leq K_p \alpha$$
    for all $\hat{s}',\hat{s}\in \hat{S}, \hat{a}_d\in \hat{A}_d$ and all $\alpha\leq (0,\frac{1}{2}\mathcal{L}]$
\end{lemma}

We now apply simulation lemma bound the difference in value for any discretised policy (i.e. restricting action space to $\hat{A}_d$) in the continuous MDP $M$ and discretised MDP $M_d$.

\begin{lemma}
\label{lemma:simulation}
    Let $M_{\hat{A}_d}=(\hat{S},\hat{A}_d,T, R,\gamma)$, that is, the continous MDP $M$ restricted to the discretised action space. And recall that discretised MDP $M_d=(\hat{S},\hat{A}_d,T_d, R_d,\gamma)$. Then for any discretised policy $\pi:\hat{S}\rightarrow \hat{A}_d$,
    $$\|V_{M_{\hat{A}_d}}^\pi - V_{M_d}^\pi \|_{\infty}=\|V_{M}^\pi - V_{M_d}^\pi \|_{\infty}\leq \frac{\mathcal{L_R}\alpha}{1-\gamma}+\frac{\gamma K_p \alpha }{(1-\gamma)^2}$$
\end{lemma}
\begin{proof}
    Lemma \ref{lm:discretised_transition} gives the bound for difference in transition probability $$\epsilon_P=\max_{\hat{s},\hat{a}_d}\|T_d(\hat{s}'|\hat{s},\hat{a}_d)-T(\hat{s}'|\hat{s},\hat{a}_d)\|_{1} =\max_{\hat{s},\hat{a}_d}2TV(T_d(\cdot|\hat{s},\hat{a}_d),T(\cdot|\hat{s},\hat{a}_d))\leq 2\max_{\hat{s}'}|T_d(\hat{s}'|\hat{s},\hat{a}_d)-T(\hat{s}'|\hat{s},\hat{a}_d)| \leq  2K_p \alpha$$
We upper bound the difference in reward using our Lipschitz assumption:
    \begin{align*}
        \epsilon_R=&\max_{\hat{s}, \hat{a}_d} |R_d(\hat{s}, \hat{a}_d) - R(\hat{s}, \hat{a}_d)|\\
        =& \max_{\hat{s}, \hat{a}_d} |R(\hat{s}_d, \hat{a}_d) - R(\hat{s}, \hat{a}_d)|\\
        \leq & \mathcal{L_R}\|\hat{s}_d-\hat{s}\|_{\infty} \\
        =& \mathcal{L_R} \alpha
    \end{align*}

The bound holds by applying simulation lemma to $M_{\hat{A}_d}$ and $M_d$ with $\epsilon_P$ and $\epsilon_R$ above:
\begin{align*}
    \|V_{M_{\hat{A}_d}}^\pi - V_{M_d}^\pi \|_{\infty}=\|V_{M}^\pi - V_{M_d}^\pi \|_{\infty}\leq \frac{\mathcal{L_R}\alpha}{1-\gamma}+\frac{\gamma K_p \alpha}{(1-\gamma)^2}
\end{align*}
\end{proof}


\subsection{Case I}
\begin{theorem} 
Suppose that $0 \leq \varepsilon<\frac{1}{1-\gamma}$ and $0 \leq \delta<1$ are two real numbers. Let M be any continuous meta-MDP with inner stochastic game $G = \langle S, A, T_\text{inner}, R_\text{inner} \rangle$.
Let us denote $M_d=\langle \hat{S}_d, \hat{A}_d, T_d, R_d, \gamma\rangle$ as discretised version of $M$ (as described in Case I) using grid size of $\lambda$. There exists inputs $m=m\left(\frac{1}{\varepsilon}, \frac{1}{\delta}\right)$ and $\varepsilon_1$, satisfying 

$$m\left(\frac{1}{\varepsilon}, \frac{1}{\delta}\right)=\tilde{O}\left(\frac{\left(\frac{2 \sqrt{n|S||A|}} {\lambda} + 1 \right)^{n|S||A|}}{\varepsilon^2(1-\gamma)^4}\right)$$

and $\frac{1}{\varepsilon_1}=O\left(\frac{1}{\varepsilon}\right)$, such that if $R\text{-}MAX$ is executed on $M$ with inputs $m$ and $\varepsilon_1$, then the following holds. Let $\mathcal{A}_t$ denote $R\text{-}MAX$'s policy at time $t$ and $\hat{s}_t$ denote the state at time $t$. With probability at least $1-\delta$, 

$$ V_{M}^*\left(\hat{s}_t\right) - V_{M}^{\mathcal{A}_t}\left(\hat{s}_t\right) \leq \varepsilon + \frac{\mathcal{K} \lambda}{(1 - \hat{\gamma})^2} + \frac{\mathcal{L_R}\alpha}{1-\gamma}+\frac{\gamma K_p \alpha }{(1-\gamma)^2}$$

is true for all but
$$
\tilde{O}\left(\frac{\left(\frac{2 \sqrt{n|S||A|}} {\lambda} + 1 \right)^{2n|S||A|} \left(\frac{2 \sqrt{|S||A|}} {\lambda} + 1 \right)^{|S||A|}} {\varepsilon^3(1-\gamma)^6}\right)
$$
timesteps.

\end{theorem}

\begin{proof}
The result follows from adding Theorem \ref{thm:pac1}, \ref{lm:discretised_transition}, and Lemma \ref{lemma:simulation}.
\end{proof}

\subsection{Case II}

\begin{theorem} 
Suppose that $0 \leq \varepsilon<\frac{1}{1-\gamma}$ and $0 \leq \delta<1$ are two real numbers. Let M be any continuous meta-MDP with inner stochastic game $G = \langle S, A, T_\text{inner}, R_\text{inner} \rangle$.
Let us denote $M_d=\langle \hat{S}_d, \hat{A}_d, T_d, R_d, \gamma\rangle$ as discretised version of $M$ (as described in Case I) using grid size of $\lambda$. There exists inputs $m=m\left(\frac{1}{\varepsilon}, \frac{1}{\delta}\right)$ and $\varepsilon_1$, satisfying 

$$m\left(\frac{1}{\varepsilon}, \frac{1}{\delta}\right)=\tilde{O}\left(\frac{(|S||A|)^{nh}}{\varepsilon^2(1-\gamma)^4}\right)$$

and $\frac{1}{\varepsilon_1}=O\left(\frac{1}{\varepsilon}\right)$, such that if $R\text{-}MAX$ is executed on $M$ with inputs $m$ and $\varepsilon_1$, then the following holds. Let $\mathcal{A}_t$ denote $R\text{-}MAX$'s policy at time $t$ and $\hat{s}_t$ denote the state at time $t$. With probability at least $1-\delta$, 

$$ V_{M}^*\left(\hat{s}_t\right) - V_{M}^{\mathcal{A}_t}\left(\hat{s}_t\right) \leq \varepsilon + \frac{\mathcal{K} \lambda}{(1 - \hat{\gamma})^2} + \frac{\mathcal{L_R}\alpha}{1-\gamma}+\frac{\gamma K_p \alpha}{2(1-\gamma)^2}$$

is true for all but
$$
\tilde{O}\left(\frac{(|S||A|)^{2nh} \left(\frac{2 \sqrt{|S||A|}} {\lambda} + 1 \right)^{|S||A|}} {\varepsilon^3(1-\gamma)^6}\right)
$$
timesteps.

\end{theorem}

\begin{proof}
The steps are similar to Case I.
\end{proof}

\newpage
\section{Experiment Details}

\subsection{Matching Pennies Table Summary}
\begin{table}[!htb]
\centering
\begin{tabular}{c|cc}
Player 1\textbackslash Player 2 & Head     & Tail     \\ \hline
Head                            & (+1, -1) & (-1, +1) \\ 
Tail                            & (-1, +1) & (+1, -1) \\ 
\end{tabular}
\label{tb:mp_appendix}
\caption{Payoff Matrix for MP}
\end{table}

\subsection{Implementation Details}
The opponent is a standard Q-learning agent who updates the Q-values at every meta-step and selects an action that corresponds to the maximum Q-value at a given state. To enable better empirical performance, the meta-agent uses Boltzmann sampling instead of greedy sampling to sample the next action from the Q-value table.
We use a discount factor of $0.8$.
For each run, we run a total of  $10 \times m \times |\hat{\mathcal{S}}|$ R-FOS iterations, where 10 is our chosen hyper-parameter.

\end{document}